\documentclass[nohyperref]{article}
\usepackage[dvipsnames]{xcolor}
\usepackage[accepted]{icml2022}

\usepackage{amsthm}
\usepackage{mathtools}
\usepackage{amsmath}
\usepackage{bbm}
\usepackage{amsfonts}
\usepackage{amssymb}

\allowdisplaybreaks

\usepackage[colorlinks=true, linkcolor=blue, citecolor=blue]{hyperref}
\usepackage{microtype}

\usepackage{enumitem}

\theoremstyle{definition}  %Sets style of subsequent newtheorems to 'definition'
\newtheorem{assumption}{Assumption}

\theoremstyle{plain}

\newtheorem{corollary}{Corollary} 
\newtheorem{lemma}{Lemma}
\newtheorem{fact}{Fact}
\newtheorem{definition}{Definition}

\usepackage{prettyref}
\newcommand{\pref}[1]{\prettyref{#1}}

\newcommand{\savehyperref}[2]{\texorpdfstring{\hyperref[#1]{#2}}{#2}}
\newrefformat{eq}{\savehyperref{#1}{\textup{(\ref*{#1})}}}
\newrefformat{eqn}{\savehyperref{#1}{Equation~\ref*{#1}}}
\newrefformat{con}{\savehyperref{#1}{Conjecture~\ref*{#1}}}
\newrefformat{lem}{\savehyperref{#1}{Lemma~\ref*{#1}}}
\newrefformat{def}{\savehyperref{#1}{Definition~\ref*{#1}}}
\newrefformat{line}{\savehyperref{#1}{line~\ref*{#1}}}
\newrefformat{thm}{\savehyperref{#1}{Theorem~\ref*{#1}}}
\newrefformat{cor}{\savehyperref{#1}{Corollary~\ref*{#1}}}
\newrefformat{sec}{\savehyperref{#1}{Section~\ref*{#1}}}
\newrefformat{app}{\savehyperref{#1}{Appendix~\ref*{#1}}}
\newrefformat{ass}{\savehyperref{#1}{Assumption~\ref*{#1}}}
\newrefformat{ex}{\savehyperref{#1}{Example~\ref*{#1}}}
\newrefformat{fig}{\savehyperref{#1}{Figure~\ref*{#1}}}
\newrefformat{alg}{\savehyperref{#1}{Algorithm~\ref*{#1}}}
\newrefformat{rem}{\savehyperref{#1}{Remark~\ref*{#1}}}
\newrefformat{conj}{\savehyperref{#1}{Conjecture~\ref*{#1}}}
\newrefformat{prop}{\savehyperref{#1}{Proposition~\ref*{#1}}}
\newrefformat{proto}{\savehyperref{#1}{Protocol~\ref*{#1}}}
\newrefformat{prob}{\savehyperref{#1}{Problem~\ref*{#1}}}
\newrefformat{claim}{\savehyperref{#1}{Claim~\ref*{#1}}}

\DeclarePairedDelimiter{\brk}{[}{]}
\DeclarePairedDelimiter{\crl}{\{}{\}}
\DeclarePairedDelimiter{\prn}{(}{)}
\DeclarePairedDelimiter{\nrm}{\|}{\|}

\let\Pr\undefined

\DeclareMathOperator{\Pr}{Pr}

\DeclareMathOperator*{\argmin}{arg\,min} % * Places subscript directly under operator
\DeclareMathOperator*{\argmax}{arg\,max}

\newcommand{\indicator}[1]{\mathbbm{1}\crl*{#1}}    %Indicator

\newcommand{\eps}{\epsilon}

\newcommand{\ldef}{\vcentcolon=}

\newcommand{\wt}[1]{\widetilde{#1}}
\newcommand{\wh}[1]{\widehat{#1}}

\def\ddefloop#1{\ifx\ddefloop#1\else\ddef{#1}\expandafter\ddefloop\fi}
\def\ddef#1{\expandafter\def\csname bb#1\endcsname{\ensuremath{\mathbb{#1}}}}
\ddefloop ABCDEFGHIJKLMNOPQRSTUVWXYZ\ddefloop
\def\ddefloop#1{\ifx\ddefloop#1\else\ddef{#1}\expandafter\ddefloop\fi}
\def\ddef#1{\expandafter\def\csname b#1\endcsname{\ensuremath{\mathbf{#1}}}}
\ddefloop ABCDEFGHIJKLMNOPQRSTUVWXYZ\ddefloop
\def\ddef#1{\expandafter\def\csname c#1\endcsname{\ensuremath{\mathcal{#1}}}}
\ddefloop ABCDEFGHIJKLMNOPQRSTUVWXYZ\ddefloop
\def\ddef#1{\expandafter\def\csname h#1\endcsname{\ensuremath{\widehat{#1}}}}
\ddefloop ABCDEFGHIJKLMNOPQRSTUVWXYZabcdefghijklmnopqrsuvwxyz\ddefloop    % Not defined for t. 
\def\ddef#1{\expandafter\def\csname hc#1\endcsname{\ensuremath{\widehat{\mathcal{#1}}}}}
\ddefloop ABCDEFGHIJKLMNOPQRSTUVWXYZ\ddefloop
\def\ddef#1{\expandafter\def\csname t#1\endcsname{\ensuremath{\widetilde{#1}}}}
\ddefloop ABCDEFGHIJKLMNOPQRSTUVWXYZ\ddefloop
\def\ddef#1{\expandafter\def\csname tc#1\endcsname{\ensuremath{\widetilde{\mathcal{#1}}}}}
\ddefloop ABCDEFGHIJKLMNOPQRSTUVWXYZ\ddefloop

\usepackage{thmtools, thm-restate}

\newcommand\numberthis{\addtocounter{equation}{1}\tag{\theequation}}

\newcommand{\myopic}{\mathsf{expl}}  
\newcommand{\lerad}{c}
\newcommand{\legap}{\alpha}

\newcommand{\regret}{\operatorname{Reg}}

\usepackage[ruled, vlined, linesnumbered]{algorithm2e}
\usepackage{cancel}
\usepackage[normalem]{ulem}

\renewcommand{\epsilon}{\varepsilon}
\renewcommand{\eps}{\varepsilon}

\newcommand{\dif}{\mathrm{d}}
\usepackage{enumitem}

\icmltitlerunning{Performance Guarantees for Epsilon-Greedy RL}

\newtheorem{innercustomlemma}{Lemma}
\newenvironment{customlemma}[1]
  {\renewcommand\theinnercustomlemma{#1}\innercustomlemma}
  {\endinnercustomlemma}

\begin{document}

\twocolumn[
\icmltitle{Guarantees for Epsilon-Greedy Reinforcement Learning\\ with Function Approximation}

\icmlsetsymbol{equal}{*}

\begin{icmlauthorlist}
\icmlauthor{Christoph Dann}{google}
\icmlauthor{Yishay Mansour}{google,tau}
\icmlauthor{Mehryar Mohri}{google,nyu}
\icmlauthor{Ayush Sekhari}{cornell}
\icmlauthor{Karthik Sridharan}{cornell}
\end{icmlauthorlist}

\icmlaffiliation{google}{Google Research}
\icmlaffiliation{tau}{Tel Aviv University}
\icmlaffiliation{nyu}{Courant Institute of Mathematical Sciences}
\icmlaffiliation{cornell}{Cornell University}

\icmlcorrespondingauthor{Christoph Dann}{cdann@cdann.net}

\icmlkeywords{Machine Learning, ICML}

\vskip 0.3in
]

\printAffiliationsAndNotice{}  

\begin{abstract}

Myopic exploration policies such as $\epsilon$-greedy, softmax, or
Gaussian noise fail to explore efficiently in some reinforcement learning tasks and yet, they
perform well in many others. In fact, in practice, they are often
selected as the top choices, due to their simplicity.
But, for what tasks do such policies succeed? Can we give theoretical
guarantees for their favorable performance? These crucial questions
have been scarcely investigated, despite the prominent practical
importance of these policies.
This paper presents a theoretical analysis of such policies and
provides the first regret and sample-complexity bounds for
reinforcement learning with myopic exploration. Our results apply to
value-function-based algorithms in episodic MDPs with bounded
Bellman Eluder dimension.  We propose a new complexity measure  called myopic exploration gap, denoted by \(\alpha\), that captures a structural property of the
MDP, the exploration policy and the given value function class. We show that the sample-complexity of myopic
exploration scales quadratically with the inverse of this quantity, $1 / \alpha^2$.
We further demonstrate through concrete examples that myopic exploration
gap is indeed favorable in several tasks where myopic
exploration succeeds, due to the corresponding dynamics and reward
structure.
\end{abstract}

\section{Introduction}

Remarkable empirical and theoretical advances have been made in scaling up Reinforcement Learning (RL) approaches to complex problems with rich observations. On the theory side, a suite of algorithmic and analytical tools have been developed for provably sample-efficient RL in MDPs with general structural assumptions \citep{jiang2017contextual, sun2018model, wang2020reinforcement, jin2020provably, foster2020instance, dann2021provably, du2021bilinear}. The key innovation in these works is to enable strategic exploration in combination with general function approximation, leading to strong worst-case guarantees.
On the empirical side, agents with large neural networks have been trained, for example, to successfully play Atari games \citep{mnih2015human}, beat world-class champions in Go \citep{silver2017mastering} or control real-world robots \citep{kalashnikov2018qt}. However, perhaps surprisingly, many of these empirically successful approaches, do not rely on strategic exploration but instead  perform simple myopic exploration. Myopic approaches simply  perturb the actions prescribed by the current estimate of the optimal policy, for example by taking a uniformly random action with probability $\epsilon$ (called $\epsilon$-greedy exploration).

While myopic exploration is known to have exponential sample complexity in the worst case \citep{osband2019deep}, it still remains a popular choice in practice. The reason for this is because myopic exploration is easy to implement and works well in a range of problems, including many of practical interest \citep[][]{mnih2015human, kalashnikov2018qt}. On the other hand, it is unknown how to implement existing strategic exploration approaches computationally efficiently with general function approximation. They either require solving intricate non-convex optimization problems \citep{jiang2017contextual, jin2021bellman, du2021bilinear} or sampling from intricate distributions \citep{zhang2021feel}.

Despite its empirical importance, theoretical analysis of myopic exploration is still quite rudimentary, and performance guarantees for sample complexity or regret bounds are largely unavailable. In order to bridge this gap between theory and practice, in this work we investigate:
\vspace{-1mm}
\begin{quote}
    \emph{In which problems does myopic exploration lead to sample-efficient learning, and can we provide a theoretical guarantee for myopic reinforcement learning algorithms such as $\epsilon$-greedy?}
\end{quote}
\vspace{-1mm}
We address this question by providing a framework for analyzing value-function based RL with myopic exploration. Importantly, the algorithm we analyze is easy to implement since it only requires minimizing standard square loss on the value function class for which many practical approaches exist, even on complex neural networks \citep{mnih2015human, silver2016mastering, silver2017mastering, rakhlin2015online}. 

Our main contributions are threefold:
\begin{enumerate}[label=\(\bullet\), topsep=0pt, itemsep=0mm]  
    \item We propose a new complexity measure called myopic exploration gap $\alpha(f, \cF)$ that captures how easy it is for a given myopic exploration strategy to identify that a candidate value function $f \in \cF$ is sub-optimal. This complexity measure is large for problems with favorable transition dynamics and rewards, where we expect myopic exploration to perform well. 
    \item We derive a sample complexity upper bound for RL with myopic exploration. We show that for any sub-optimal $\cF' \subset \cF$, the algorithm uses policies corresponding to $\cF'$ for at most $\tilde O\prn[\big]{\tfrac{H^2 d}{\alpha(\cF', \cF)^2}}$ episodes, where $H$ is the episode length, $d$ is the Bellman Eluder dimension and $\alpha(\cF', \cF) = \min_{f \in \cF'} \alpha(f,\cF)$. Using this  sample complexity bound, we provide the first regret bound for $\epsilon$-greedy RL in MDPs. 
    \item We prove an almost matching sample complexity lower bound of $\Omega(d / \alpha(\cF', \cF)^2)$ for $\epsilon$-greedy RL, hence showing that the dependency on Bellman Eluder dimension or the myopic exploration gap in our upper bound cannot be improved further. 
\end{enumerate}

\section{Preliminaries}
\label{sec:setting}
\textbf{Episodic MDP.} We consider RL in episodic Markov Decision Processes (MDPs), each denoted by $\cM = (\cX, \cA, H, P, R)$, where $\cX$ is the observation (or state) space, $\cA$ is the action space, $H \in \bbN$ is the number of time steps in each episode, $P = (P_h)_{h \in [H]}$ is the collection of transition kernels and $R = (R_h)_{h \in [H]}$ is the collection of immediate reward distributions. An episode is a sequence of states, actions and rewards $(x_1, a_1, r_1, \dots, x_H, a_H, r_H, x_{H+1})$ with a fixed initial state $x_1 = x_{\textrm{init}} \in \cX$ and terminal state $x_{H+1} = x_{\textrm{end}} \in \cX$. 
All states (except $x_1$) are sampled from the transition kernels $x_{h+1} \sim P_h(\cdot | x_h, a_h) = P_h(x_h, a_h)$ and rewards are generated as $r_h \sim R_h(x_h, a_h)$. We assume that $P$ and $R$ are s.t. $r_h \geq 0$ and the sum of all rewards in an episode is bounded  as $\sum_{h=1}^H r_h \leq 1$ for any action sequence almost surely.

\textbf{Policies and value functions.} The agent interacts with the environment in multiple episodes indexed by $t$. It chooses actions according to a \emph{policy} $\pi = (\pi_h \colon \cX \rightarrow \Delta_{\cA})_{h \in [H]}$ that maps states to distributions over actions. The space of all such policies is denoted by $\Pi$. 
We denote the distribution over an episode induced by following a policy $\pi$ by $\bbP_{\pi}$ and the expectation w.r.t. this law by $\bbE_{\pi}$. We define the Q-function and value-function of a policy $\pi$ at step $h$ as
\begin{align*}
    Q^\pi_h(x, a) &= \bbE_{\pi}[r_h + V_{h+1}^{\pi}(x_{h+1}) ~|~x_h = x, a_h = a],\\
    V_h^\pi(x) &= \bbE_{\pi}[Q_h^\pi(x_h, a_h) | x_h = x], 
\end{align*}
respectively. The optimal Q- and value functions are denoted by $Q^\star_h$ and $V^\star_h$. The occupancy measure of $\pi$ at time $h$ is denoted as $\mu_h^\pi(x,a) = \bbP_{\pi}(x_h = x, a_h = a)$.
For any function $f \colon \cX \times \cA \rightarrow \bbR$, the \emph{Bellman operator} at step $h$ is
\begin{align*}
    (\cT_h f)(x,a) &= \bbE\brk[\big]{r_h \!+ \!\max_{a' \in \cA} f(x_{h+1}, a')  \Big| x_h \!= \!x, a_h \!=\! a}.
\end{align*} 

\textbf{Value function class.}
We assume the learner has access to a Q-function class $\cF = \cF_1 \times \dots \times \cF_{H}$ where $\cF_h \subseteq (\cX \times \cA \rightarrow [0, 1])$, and for convention we set $\cF_{H+1} = \{ f_{H+1} = 0\}$. 
We denote by $\pi^f = \{ \pi^f_{h} \}_{h \in [H]}$ the greedy policy w.r.t. $f \in \cF$, i.e., $\pi^f_{h}(x) = \argmax_{a \in \cA} f_h(x, a)$. The set of greedy policies of $\cF$ is denoted by $\Pi_{\cF}$.
The Bellman residual and squared Bellman error of $f \in \cF$ at time $h$ are 
\begin{align*}
    \cE_h f = f_h - \cT_h f_{h+1} \quad \textrm{ and } \quad \cE_h^2 f = (f_h - \cT_h f_{h+1})^2.
\end{align*}

We further adopt the following assumption common in RL theory with function approximation \citep[cf.][]{jin2021bellman, wang2020reinforcement, du2021bilinear, dann2021provably}:
\begin{assumption}[Realizability and Completeness]\label{ass:completeness}
$\cF$ is realizable and complete under the Bellman operator, that is, $Q^\star_h \in \cF_h$ for all $h \in [H]$ and for every $h \in [H]$ and $f_{h+1} \in \cF_{h+1}$ there is a $f_h \in \cF_{h}$ such that $f_{h} = \cT_h f_{h+1}$.
\end{assumption}

\textbf{Bellman Eluder dimension.} To capture the complexity of the MDP and the function class we use the notion of Bellman Eluder dimension $\operatorname{dim}_{\operatorname{BE}}(\cF, \Pi_{\cF}, \mu)$ introduced by \citet{jin2021bellman}. This complexity measure is a distributional version of the Eluder dimension applied to the class of Bellman residuals $\cE_h \cF$.
The Bellman Eluder dimension matches the natural complexity measure in many special cases, e.g., number of states and actions ($SA$) in tabular MDPs or the intrinsic dimension $d$ in linear MDPs.
The formal definition and the main properties of the Bellman Eluder dimension we use in our work, can be found in \pref{app:supporting_results}.

We denote by $N_{\cF_h}(\lambda)$ the $\ell_\infty$ covering number of $\cF_h$ w.r.t. radius $\lambda$, i.e., the size of the smallest $\cZ \subset \cF_h$ so that for any $f \in \cF$ there is a $f' \in \cZ$ with $\max_{h \in [H]}\|f_h - f'_h\|_\infty \leq \lambda$ and we use $\bar N_{\cF}(\lambda) = \sum_{h=1}^H N_{\cF_h}(\lambda)$.

\section{RL with Myopic Exploration} 

Our work analyzes value-function based algorithms with myopic exploration that follow the template in \pref{alg:egreedy_genfun}. The algorithm receives as input a value function class $\cF$ as well as an exploration mapping $\myopic$ that maps each function in $\cF$ to an exploration policy $\Pi$. Before each episode, the algorithm computes a Q-function estimate $\wh f_t$ in a dynamic programming fashion by least squares regression using all data observed so far (finite horizon fitted Q iteration). The sampling policy $\wt \pi_t$ is then determined using the exploration mapping $\myopic$. After sampling one episode with $\wt \pi_t$, the algorithm adds all observations to the datasets and computes a new value function.

This algorithm template encompasses many common exploration heuristics by setting the exploration mapping $\myopic$ appropriately. These include:\footnote{Our examples are Markovian policies but we also support non-Markovian exploration policies that inject temporally correlated noise, e.g., the Ornstein-Uhlenbeck noise used by \citet{lillicrap2015continuous} or $\epsilon$-greedy with options \citep{dabney2020temporally}.}
\begin{center}
\begingroup
\setlength{\tabcolsep}{10pt}
\renewcommand{\arraystretch}{1.3} 
\begin{tabular}{ c | c }
 Exploration strategy & $[\myopic(f)]_h(a | x) \propto$ \\ \hline\hline
 $\epsilon$-greedy & $(1 - \epsilon) \pi_h^f(a |x) + {\epsilon}/{A}$ \\  
 softmax & $ \exp(\beta f_h(x, a))$\\
 Gaussian noise & $   
   \exp\prn[\big]{- \frac{1}{2\sigma^2} \prn[\big]{a - \pi^f_h(x)}^2}$
\end{tabular}
\endgroup
\end{center}
Note that $\myopic$ only takes the point estimate for the Q-function $\wh f_t$ as input and no measure of uncertainty. Thus,  approaches like upper-confidence bounds or Thompson sampling are beyond the scope of our work. Furthermore, to keep the exposition clean, we restrict ourselves to mappings $\myopic$ that are independent of the number of episodes $k$, but we allow the mapping to depend on the total number of episodes $T$. Getting anytime guarantees with variable exploration policies (e.g. $\epsilon$-greedy with decreasing $\epsilon$) is an interesting direction for future work.

One important feature of \pref{alg:egreedy_genfun} is that it only requires solving least-squares regression problems, for which numerous efficient methods exist, even when $\cF$ are large neural networks. This is in stark contrast to all known algorithms for general function approximation with worst-case sample-efficiency guarantees. Those require either solving an intricate non-convex optimization problem to ensure global optimism \citep{jiang2017contextual,dann2018oracle, jin2021bellman, du2021bilinear} or sampling from a posterior distribution without closed form \citep{dann2021provably}. The lack of computationally and statistically worst-case efficient algorithms is a strong motivation for studying the sample-efficiency of the simple and computationally tractable approach in \pref{alg:egreedy_genfun}.

We assume that the regression problem in  \pref{line:lsregression} is solved exactly for ease of presentation but our analysis can be easily extended  to allow small optimization error in each episode. 

\begin{algorithm}[t]
\label{alg:egreedy_genfun}
\SetNoFillComment
\SetInd{0.7em}{0.5em}
\SetKwInOut{Inputa}{input}
\SetKwInOut{Outputa}{ouput}
\SetKwInOut{Returna}{return}
  \SetKwProg{Fn}{function}{:}{}
  \SetNlSty{bfseries}{\color{black}}{}
\Inputa{function class $\cF = \cF_1 \times \dots \times \cF_{H+1}$}
\Inputa{myopic exploration mapping $\myopic \colon \cF \rightarrow \Pi$}
Initialize datasets $\cD_h \gets \varnothing$ for all $h \in [H]$\;
\For{episode $t = 1, 2, \dots$}{
Set $\wh f_{t, H+1} = 0$\;
\For{$h = H, \dots, 1$}{
     Fit Q-function with least-squares regression $\wh f_{t, h} \gets \argmin_{f \in \cF_h} \cL_h(f, \wh f_{t, h+1})$ where \label{line:lsregression}
    \begin{align*}
        \cL_h(f, f') = \hspace{-5mm}\sum_{(x,a,r, x') \in \cD_h}\hspace{-5mm} \prn[\Big]{f(x,a) - r - \max_{a'} f'(x', a')\!}^2
    \end{align*} 
}
Set myopic exploration policy $\wt \pi_t \gets \myopic(\wh f_t)$\;
 Sample one episode $\{x_1, a_1, r_1, \dots x_{H+1}\}$ with $\wt \pi_t$\;
Add observations to datasets $\cD_{h} \gets \cD_h \cup \{(x_h, a_h, r_h, x_{h+1})\}$ for all $h \in [H]$\;
}
\caption{RL with myopic exploration} 
\end{algorithm}

\section{Myopic Exploration Gap}
In this section, we introduce a new complexity measure called \textit{myopic exploration gap} in order to formalize the sample-efficiency of myopic exploration. One important feature of this quantity is that it depends both on the dynamics and the reward structure of the MDP. This enables us to give tighter guarantees in the presence of favorable rewards, a key property absent in existing analyses of myopic exploration.

The sample complexity and regret of any (reasonable) algorithm depends on a notion of gap or suboptimality, either of individual actions or entire policies, as shown by existing works on provably sample efficient RL \citep{simchowitz2019non, dann2021beyond, wagenmaker2021beyond}. Intuitively, these algorithms need to test whether a candidate value function $f \in \cF$ is optimal, and typically value functions that have large instantaneous regret $V^{\star}(x_1) - V^{\pi^f}(x_1)$ require fewer samples to be ruled out as being the optimal. This is also the case for myopic exploration based algorithms but the suboptimality gap alone is not sufficient to quantify their performance. This is because:
\begin{enumerate}[label=\((\alph*)\), topsep=0pt, itemsep=0mm]     \item Myopic exploration policies such as $\epsilon$-greedy mostly collect samples  in the ``neighborhood" of the greedy policy. However, the optimal policy may visit high-reward state-action pairs that are outside of this neighborhood and thus have low probability of being visited under the exploration policy. In such cases, to discover the high reward states and to identify the gap, the agent thus needs to collect a very large number of episodes. Essentially, the effective size of the gap is reduced.
    \item The agent is not required to estimate the suboptimality gap w.r.t. the optimal policy \(\pi^\star\) in order to rule out a candidate policy $\pi$ -- it only needs to find a better policy $\pi'$. If there is a such a $\pi'$ in the ``neighborhood" of $\pi$, then the exploration policy may be more effective at identifying the optimality gap  between $\pi$ and $\pi'$ even though it is smaller than the optimality gap between $\pi$ and $\pi^\star$.
\end{enumerate}
We address both of these issues in our definition of myopic exploration gap for $f \in \cF$ by (1) considering the gap $V_1^{\pi'}(x_1) - V_1^{\pi^f}(x_1)$ between $\pi^f$ and any other policy $\pi'$, and (2) normalizing this gap by a scaling factor $c$. This factor, which we call \emph{myopic exploration radius}, measures the effectiveness of the exploration policy $\myopic(f)$ at collecting informative samples for identifying the gap. 

\begin{definition}[myopic exploration gap]\label{def:local_exploration_gap}
Given a function class $\cF$, exploration mapping $\myopic \colon \cF \rightarrow \Pi$, MDP $\cM$ and policy class $\Pi'$, we define the myopic exploration gap $\legap(f, \cF, \Pi',  \myopic, \cM)$ of $f \in \cF$ as the value of
\begin{align}
\begin{split}
\label{eq:local_gap} 
    \sup_{\pi' \in \Pi', c \geq 1} \frac{1}{\sqrt{c}} (V_1^{\pi'}(x_1)& - V_1^{\pi^f}(x_1))\\ 
    \qquad \textrm{such that for all } & f' \in \cF \textrm{ and }h \in [H],\\
    \bbE_{\pi'}[(\cE_h^2 f')(x_h, a_h)] 
    &\leq c
    \bbE_{\myopic(f)}[(\cE_h^2 f')(x_h, a_h)] \\
    \bbE_{\pi^f}[(\cE_h^2 f')(x_h, a_h)] 
    &\leq c  
    \bbE_{\myopic(f)}[(\cE_h^2 f')(x_h, a_h)].
    \end{split} 
\end{align}
Furthermore, we define the myopic exploration radius $\lerad(f, \cF, \Pi',  \myopic, \cM)$ as the smallest value of $c$ that attains the maximum in $\eqref{eq:local_gap}$. To simplify the  notation, we omit the dependence on $\myopic$ and $\cM$. When $\Pi' = \Pi_{\cF}$ we also omit the dependence on $\Pi'$ and use the notation $\legap(f, \cF)$ and $\lerad(f, \cF)$ respectively. 
\end{definition}

The constraints in  \pref{eq:local_gap} determine how small the normalization of the gap can be chosen. We compare the expected squared Bellman errors $\cE_h^2 f'$ under different distributions. To see why, consider the following: least squares value iteration in \pref{alg:egreedy_genfun} minimizes the squared Bellman error on the dataset collected by exploration policies, which is closely related to the RHS in the constraints. Additionally, our analysis relates the LHS of the constraints to the gap\(-\)if $\cE_h^2 f$ is small for all $h$ under for both $\bbE_{\pi'}$ and $\bbE_{\pi^f}$, then the gap $V_1^{\pi'}(x_1) - V_1^{\pi^f}(x_1)$ cannot be large. Therefore, $c$ measures how effective the exploration policy is at identifying the gap between $\pi^f$ and $\pi'$.

\paragraph{Alternative Form in Tabular MDPs.} The definition of myopic exploration gap in \pref{eq:local_gap} only involves expectations w.r.t. policies induced by $\cF$, which is desirable in the rich observation setting. However, in tabular MDPs where the number of states $S$ and actions $A$ is finite and $\cF$ is the entire Q-function class with $N_{\cF}(\lambda) \approx (A/\lambda)^{SAH}$, it may be more convenient to have constraints on individual state-action pairs instead. Since the class of squared Bellman errors $\cE_h^2 \cF)$ for tabular representations is rich enough to include functions that are nonzero in a single state-action pair, we can write \pref{eq:local_gap} with $SAH$ constraints on the corresponding occupancy measures in tabular problems:
\begin{align}
\begin{split}
\label{eqn:local_gap_tabular}
    \sup_{\pi' \in \Pi', c \geq 1} \frac{1}{\sqrt{c}}& (V_1^{\pi'}(x_1) - V_1^{\pi^f}(x_1))\\ 
    \textrm{such that for all } & (x, a) \in \cX \times \cA \textrm{ and }h \in [H],\\
    \mu^{\pi'}_{h}(x,a) &\leq c \mu^{\myopic(f)}_{h}(x, a) \\
    \mu^{\pi^f}_{h}(x,a) &\leq c \mu^{\myopic(f)}_{h}(x, a). 
    \end{split}
\end{align}

While determining the exact value of $\alpha(f, \cF)$ is challenging due to its intricate dependence on the MDP dynamics, chosen function class and the exploration policy, we can often provide meaningful and interpretable bounds.

\section{When~are~Myopic~Exploration~Gaps~Large?}
\label{sec:local_exploration_gap} 
We now discuss various structural properties of the transition dynamics and rewards under which the myopic exploration gap is large, and thus myopic exploration converges quickly. 
\subsection{Favorable Transition Dynamics}
Structure in the transition dynamics can make myopic exploration approaches effective. To understand how effective is the myopic exploration gap for those cases, first note that $\alpha(f, \cF)$ can never be larger than the optimality gap of $\pi^f$ (by definition). However, as shown in the following lemma, it is also lower bounded when the state-space distribution induced by $\myopic(f)$ is close to that of any other policy.
\begin{restatable}{lemma}{legapgeneral}
\label{lem:legap_general_bound} 
When $Q^\star \in \cF$, the myopic exploration gap is bounded for all $f \in \cF$ as
\begin{align*}
\prn[\Big]{\max_{\pi' \in \Pi'} \nrm[\Big]{ \frac{\bbP_{\pi'}}{\bbP_{\myopic(f)}}}_{\infty}}^{-\frac{1}{2}}
\leq 
\frac{\legap(f, \cF)}
{V_1^\star(x_1) - V^{\pi^f}_1(x_1)} \leq 1. 
\end{align*} Furthermore, the myopic exploration radius is bounded as $\lerad(f, \cF)
\leq \max_{\pi' \in \Pi'} \left\| \frac{\bbP_{\pi'}}{\bbP_{\myopic(f)}}\right\|_{\infty}$.
\end{restatable} 
Similar dependencies on the Radon-Nikodym derivative ${\bbP_{\pi'}}/{\bbP_{\myopic(f)}}$ can also be found in the form of the concentrability coefficient in error bounds for fitted Q-iteration under a single behavior policy \citep{antos2008learning}. 
The lower-bound in \pref{lem:legap_general_bound} can be used to bound the myopic exploration gap under various other assumptions, including the following worst-case lower-bound~for~$\epsilon$-greedy. 
\begin{restatable}{corollary}{coregreedygen}
\label{cor:legap_worstcase_egreedy}
For any MDP with finitely many actions with $|\cA|=A$ and $f \in \cF$, the myopic exploration gap of $\epsilon$-greedy can be bounded as
\begin{align*}
   \legap(f, \cF) \geq \left(\frac{\epsilon}{A}\right)^{\frac{H}{2}} (V^{\star}_1(x_1) - V^{\pi^f}_1(x_1))
\end{align*}
and the exploration radius is bounded as $\lerad(f, \cF) \leq \left(\frac{\epsilon}{A}\right)^H$.
\end{restatable}
We can similarly derive a lower bound on the myopic exploration gap for softmax-policies, where $\epsilon$ is replaced by $e^{-\beta}$, i.e.,  
$\legap(f, \cF) \geq (A e^\beta)^{-H/2} (V^{\star}_1(x_1) - V^{\pi^f}_1(x_1))$.

The bound in \pref{cor:legap_worstcase_egreedy} is exponentially small in $H$ which is not surprising. It is well known that $\epsilon$-greedy is ineffective in some problems and the construction in the proof of \pref{thm:lower_bound} will provide a formal example of this. However, in many problems, the gap can be much larger due to favorable dynamics. 

\subsubsection{Small Covering Length and Coverage}
\label{sec:coverlen}

We first turn to insights from prior work and show that our framework captures the favorable cases that they identified.
\citet{liu2018simple} investigated a related question to ours\(-\)when is myopic exploration sufficient for PAC learning with polynomial bounds in tabular MDPs. However, their work focuses on the infinite horizon setting and only considers explore-then-commit Q-learning where all the exploration is done using a uniformly random policy. They identify conditions under which the covering length of their exploration policy is polynomially small, a quantity that  governs the sample-efficiency of Q-learning with a fixed exploration policy \citep{even2003learning}. Covering length is typically used in infinite-horizon MDPs but we can transfer the concept to our episodic setting as well: 
\begin{definition}
The covering length $L(\pi)$ of a policy $\pi \in \Pi$ is the number of episodes until all state-action pairs have been visited at all time steps with probability at least $1/2$.
\end{definition}
Lower bounding $L(\pi)$ can be thought of as a variant of the popular coupon collector's problem where each object has a different probability. Since the agent has to obtain at least one sample from each $(x, a, h)$ triplet with probability $1/2$, and the probability to receive such a sample is $\mu^{\pi}_h(x, a)$ in each episode, the covering length must scale as
\begin{align*}
    L(\pi) &= 
    \Omega\prn[\Big]{ \frac{1}{\min_{x, a, h} \mu^{\pi}_h(x, a)}}. 
\end{align*}
Intuitively, we expect learning to be efficient if $\min_{x, a, h} \mu^{\pi}_h(x, a)$ is large for the data collection policy $\pi$, since the agent is able to collect samples from the entire state-action space. Note that   
\begin{align*}
    \max_{\pi' \in \Pi'} \left\| \frac{\bbP_{\pi'}}{\bbP_{\myopic(f)}}\right\|_{\infty}
    &\leq \left\| \frac{1}{\bbP_{\myopic(f)}}\right\|_{\infty}
    \leq \frac{1}{\min_{x, a, h} \mu^{\pi}_h(x, a)}. 
\end{align*}
Plugging the above bound in \pref{lem:legap_general_bound} we get that 
\begin{align*}
    \legap(f, \cF) &\geq \sqrt{\min_{x, a, h} \mu^{\pi}_h(x, a)} (V^{\star}_1(x_1) -  V^{\pi^f}_1(x_1))\\
    &= \Omega(L(\myopic(f))^{-1/2}) (V^{\star}_1(x_1) - V^{\pi^f}_1(x_1)).
\end{align*}
This shows that myopic exploration gap is  large  when the covering length is small, and thus our definition recovers the prior results based on covering length. 

\subsubsection{Small Action Variation} 
One expects myopic exploration to be effective in MDPs where the transition dynamics have almost identical next-state distributions corresponding to taking different actions at any given state. In such MDPs, exploration is easy as all the policies will have similar occupancy measures.

Prior works \cite{jiang2016structural, liu2018simple} have considered an additive notion of action variation that bounds the \(L_1\) distance between the next-state distributions corresponding to different actions. However, for our applications, a multiplicative version of action variation is better suited. 

\begin{definition}[Multiplicative Action Variation] \label{def:multiplicative_action_variation} For any MDP with state space \(\cX\), action space \(\cA\) and transition dynamics \(P\), define multiplicative action variation as the minimum \(\delta_P\) such that for any \(x, a, a'\) and \(h\), 
\begin{align*} 
\left\|\frac{\dif P_h(x, a)}{\dif P_h(x, a')}\right\|_\infty \leq \delta_P.
\end{align*}
\end{definition} 

The above definition implies that for any state \(x\), actions \(a\) and \(a'\), and next state \(x'\), the transition probability \(P_h(x' \mid x, a)\) is bounded by \(\delta_P P_h(x' \mid x, a')\). Thus, by symmetry we always have that \(\delta_P \geq 1\). The smaller the value of \(\delta_P\), the better we expect myopic exploration to perform. 
\begin{restatable}{lemma}{multiactionvar}
\label{lem:multiplicative_av_bound}
Let \(\cM\) be any MDP with multiplicative action variation \(\delta_P\), and \(Q^* \in \cF\). Then, for any \(f \in \cF\), the myopic exploration gap of \(\epsilon\)-greedy satisfies 
\begin{align*} 
\alpha(f, \cF) \geq \sqrt{\frac{\epsilon}{A \delta_P^H}} \cdot (V_1^{\star}(x_1) - V_1^{\pi^f}(x_1)).
\end{align*}
Further, \(\forall f \in \cF\), the exploration radius $\lerad(f, \cF) \leq {A \delta_P^H}/{\epsilon}$.
\end{restatable} 

A consequence of \pref{lem:multiplicative_av_bound} is that when the multiplicative action variation is  small, in particular when \(\delta_H \leq 1 + 1/H\), we have \(\alpha(f, \cF) \geq \sqrt{{\epsilon}/{e A}} (V_1^{\pi^\star}(x_1) - V_1^{\pi^f}(x_1))\) and thus myopic exploration converges quickly (see \pref{thm:sample_complexity}). An extreme case of this is when the next-state distributions do not depend on the chosen action at all and thus \(\delta_P = 1\). This represents the contextual bandit setting. 
\begin{corollary}[MDPs with contextual bandit structure] 
\label{cor:contextual_bandits}
Consider MDPs where actions do not affect the next-state distribution, i.e., $P_h(\cdot \mid x, a) = P_h(\cdot \mid x, a')$ for all $x \in \cX, (a, a') \in \cA^2, h \in [H]$. Then, for any \(f \in \cF\), the myopic exploration gap of $\epsilon$-greedy satisfies
\begin{align*}
   \legap(f, \cF) \geq \sqrt{\frac{\epsilon}{A}} \cdot (V^{\star}_1(x_1) - V^{\pi^f}_1(x_1)),
\end{align*} 
and the myopic exploration radius satisfies  $\lerad(f, \cF) \leq A / \epsilon$.
\end{corollary}

\subsection{Favorable Rewards}
In the previous section, we discussed several conditions under which the dynamics make an MDP conducive to efficient myopic exploration, independent of rewards. However the reward structure can also significantly impact the efficiency of myopic exploration. Empirically, this has been widely recognized \citep{mataric1994reward}. The general rule of thumb is that dense-reward problems, where the agent receives a reward signal frequently, are easy for myopic exploration, and sparse-reward problems are difficult. Our myopic exploration gap makes this rule of thumb precise and quantifies when exactly dense rewards are helpful.
We illustrate this for $\epsilon$-greedy exploration in the following example.

\subsubsection{Grid World Nagivation Example}
\begin{figure}
    \centering
    \includegraphics[width=\linewidth]{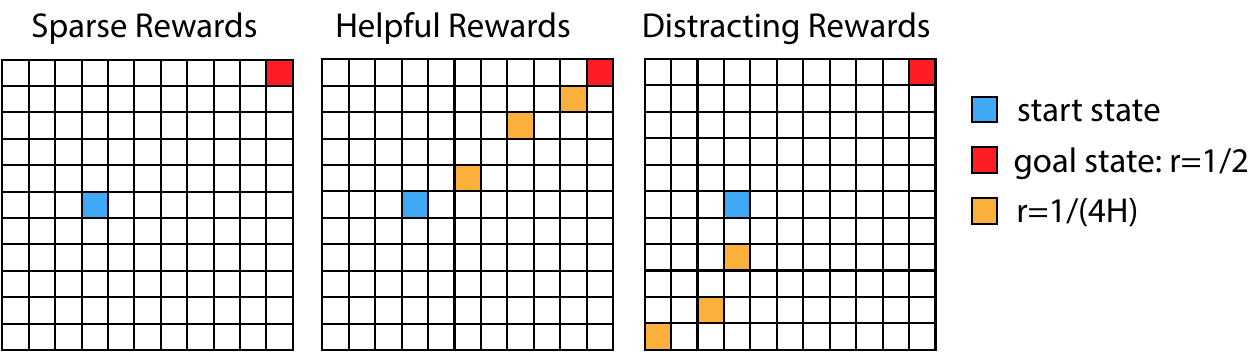}
    \caption{Three grid-world tasks with the same dynamics and optimal policy but different rewards. The optimal policy always moves from the start (blue) to the goal (red). Left: only reward of $1$ at the goal. Middle: additional small rewards help myopic exploration. Right: additional small rewards hurt myopic exploration.}
    \label{fig:effect_of_rewards}
\end{figure}

Consider a grid-world navigation task with deterministic transitions. The agent is supposed to move from the start to the goal state. \pref{fig:effect_of_rewards} depicts this task with 3 different reward functions. The horizon is chosen so that the agent can reach the goal only without any detour ($H = 13$ here).

\paragraph{Sparse Rewards.} In the version on the left, the agent only receives a non-zero reward when it reaches the goal for the first time. This is a sparse reward signal since the agent requires many time steps (roughly $H$) until it receives an informative reward. Exploration with $\epsilon$-greedy thus requires $O(A^H)$ many samples in order to find an optimal policy. Our results confirm this, since  there are many suboptimal greedy policies $\pi^f$ for which the myopic exploration gap is 
\begin{align*}
    \alpha(f, \cF) = \frac{1}{2} \left(\frac{\epsilon}{A}\right)^{H/2}.
\end{align*}

\paragraph{Helpful Dense Rewards.} In the version in the middle, someone left breadcrumb every $B$ steps along the shortest path to the goal. The agent receives a small reward of ${1}/{2H}$ each time it reaches a breadcrumb for the first time. These intermediate rewards are helpful for $\epsilon$-greedy which performs well in this problem. Our theoretical results confirm this since  any suboptimal greedy policy has a myopic exploration gap of at least
\begin{align}\label{eqn:breadcrumb_large_gaps}
\legap(f, \cF) \geq 
\frac{1}{2H} \left(\frac{\epsilon}{A}\right)^{B/2}. 
\end{align}
This bound holds since any suboptimal policy misses out on at least one breadcrumb or the goal and it could reach that and increase its return by $\geq1/2H$ by changing at most $B$ actions. In this dense-reward setting with $B \ll H$, the corresponding myopic exploration gaps are much larger.

\paragraph{Distracting Dense Rewards.} Dense rewards are not always helpful for myopic exploration. Consider the version on the right. Here, the breadcrumbs distract the agent from the goal. In particular, $\epsilon$-greedy will learn to follow the breadcrumb trail and eventually get stuck in the bottom left corner. It would then require exponentially many episodes to discover the high reward at the goal state. We can show that the myopic exploration gaps are
\begin{align*}
\legap(f, \cF)  \begin{cases}
= \frac{1}{2} \left(\frac{\epsilon}{A}\right)^{H/2}& \textrm{if $\pi^f$ reaches all breadcrumbs}\\
\geq \frac{1}{2H} \left(\frac{\epsilon}{A}\right)^{B/2} & \textrm{otherwise}
    \end{cases}
\end{align*}
for any suboptimal greedy policy $\pi^f$. Thus, there are value functions with suboptimal greedy policies that have an exponentially small gap, similar to the sparse reward setting.

\subsubsection{Favorable Rewards Through Potential-Based Reward Shaping?}
\label{sec:reward_shaping}
Potential-based reward shaping \citep{ng1999policy, grzes2017reward} is a popular technique for finding a reward definition that preserve the optimal policies of a given reward function but may be easier to learn. Formally, we here call two average reward definitions $\bar r = \{\bar r_h\}_{h \in [H]}, \bar r' = \{\bar r'_h\}_{h \in [H]}$ with $\bar r_h, \bar r_h' \colon \cS \times \cA \rightarrow [0, 1]$  a potential-based reward shaping of each other if $\bar r_h(s,a) - \bar r'_h(s,a) = \Phi_h(s) -  \bbE[\Phi(s_{h+1}) \mid x_h = x, a_h = a]$ for all $s,a, h$ where $\Phi = \{\Phi_h\}_{h \in [H]}$ is a state-based potential function with $\Phi_1(s_{\textrm{init}}) = \Phi_{H+1}(s_{\textrm{end}}) = 0$. One can show by a telescoping sum that the total return of any policies under $\bar r$ and $\bar r'_h$ is identical. As a result, since the myopic exploration gaps  in the tabular setting only depend on the rewards through the return of policies (\pref{eqn:local_gap_tabular}), the gaps are identical under both reward functions. This suggests that the efficiency of $\epsilon$-greedy exploration is not affected by potential-based reward shaping. Given the empirical success of potential-based reward shaping, this may seem surprising. However, as we illustrate with an example in \pref{app:myopic_exploration_further}, this difference in empirical performance is due to optimistic or pessimistic initialization effects and not due to $\epsilon$-greedy exploration.

\section{Theoretical Guarantees}
\label{sec:general_funapprox}
In this section, we present our main theoretical guarantees for \pref{alg:egreedy_genfun}. At the heart of our analysis is a sample-complexity bound that controls the number of episodes in which \pref{alg:egreedy_genfun} selects a poor quality value function $\wh f_t$, for which $\pi^{\wh f_t}$ has large suboptimality. In particular, we show that for any subset of function $\cF' \subset \cF$, the number of times for which \(\wh f_t \in \cF'\) is selected scales inversely with the square of the smallest myopic exploration gap $\legap(\cF', \cF) = \inf_{f \in \cF'} \legap(f, \cF)$.

Intuitively, when \(\cF'\) contains suboptimal policies and \(\legap(\cF', \cF)\) is large, then for any $f \in \cF'$ there exists a policy $\pi' \in \Pi_{\cF}$ that achieves higher return than $\pi_f$, and the exploration policy $\myopic(f)$ will quickly collect enough samples to allow \pref{alg:egreedy_genfun} to learn this difference. Thus, such an $f$ would not be selected any further. The following theorem formalizes this intuition with the sample complexity bound. 

\begin{restatable}[Sample Complexity Upper Bound]{theorem}{thmsamplecomplex}
\label{thm:sample_complexity} Let $\delta \in (0, 1)$ and \(T \in \bbN\), and suppose  \pref{alg:egreedy_genfun} is run with a function class $\cF$ that satisfies \pref{ass:completeness}. Further, let $\cF' \subseteq \cF$ be any subset of value functions. Then, with probability at least \(1 - \delta\), the number of episodes within the first $T$ episodes where $\wh f_t \in \cF'$ is selected is bounded by
\begin{align*}
O \prn[\bigg]{ 
     \frac{\ln \lerad(\cF', \cF)}{\legap(\cF', \cF)^2}
     H^2 d \ln \prn[\Big]{\frac{\bar N_{\cF}(T^{-1})\ln T}{\delta}}}.
\end{align*} 
Here, $d = \operatorname{dim}_{\operatorname{BE}}(\cF', \Pi_{\cF}, 1/\sqrt{T})$ is the Bellman-Eluder dimension of $\cF'$, and 
$$\legap(\cF', \cF) \ldef{} \inf_{f \in \cF'} \legap(f, \cF), ~~~~
    \lerad(\cF', \cF) \ldef{} \sup_{f \in \cF'} \lerad(f, \cF)$$
with $\legap(f, \cF)$ and $\lerad(f, \cF)$ defined in \pref{def:local_exploration_gap}. 
\end{restatable}

We defer the complete proof of \pref{thm:sample_complexity} to \pref{app:upper_bound_proofs} and provide a brief sketch in \pref{sec:sample_complexity_sketch}. 
Note that, although \pref{thm:sample_complexity} allows arbitrary subsets $\cF' \subset \cF$, the result is of interest only when $\cF'$ contains value functions for which the corresponding greedy policies are suboptimal. When an optimal policy $\pi^\star \in \Pi_{\cF'}$, we have that $\legap(\cF', \cF) = 0$ and thus, the sample complexity bound is vacuous. 

We next compare our result to the prior guarantees in RL with function approximation, specifically with the results of \citet{jin2020provably}. For any \(\lambda \geq 0\), if we instantiate \pref{thm:sample_complexity} with \(\cF'\) consisting of all the value functions that are not $\lambda$-optimal, i.e., $\cF' = \cF_\lambda = \{ f \in \cF \colon V^{\pi^f}(x_1) \leq V^\star(x_1) - \lambda \}$, we get that within the first  
\begin{align}
\scalebox{0.95}{
$\displaystyle
     \widetilde O\prn[\bigg]{\frac{\ln \lerad(\cF_\lambda, \cF)}{\legap(\cF_\lambda, \cF)^2}  H^2 d \ln \prn[\Big]{\bar N_{\cF}\prn[\Big]{\frac{\legap(\cF_\lambda, \cF)^2}{H^2 d\ln\lerad(\cF_\lambda, \cF)}}}\!}$}
     \label{eq:sample_complex_comparison}
\end{align}
episodes at least a constant fraction of the chosen greedy policies would be $\lambda$-optimal. Thus, terminating \pref{alg:egreedy_genfun} after collecting the above mentioned number of episodes, and returning the corresponding greedy policy of $\wh f_{t}$ for a uniformly random choice of \(t\) ensures that the output policy is $\lambda$-suboptimal with at least a constant probability. Our sample complexity bound in 
\pref{eq:sample_complex_comparison} matches the sample complexity of the \textsc{Golf} algorithm \citep{jin2020provably} in terms of its dependence on $H$ and $d$. However, our bound replaces their $\lambda$ dependency with the problem-dependent quantity ${\legap(\cF_{\lambda}, \cF)}/{\sqrt{\ln \lerad(\cF_\lambda, \cF)}}$ that captures the efficiency of the chosen exploration approach for learning (this dependence is tight as shown by the lower bound in \pref{thm:lower_bound}). The comparison is a bit subtle; while our sample complexity bound is typically larger than that of \citet{jin2020provably}, the provided algorithm is computationally efficient with running time of the order of  \pref{eq:sample_complex_comparison} whenever an efficient square loss regression oracle is available for the class \(\cF\). 
On the other hand,  \citet{jin2020provably} rely on optimistic planning which is typically computationally inefficient. 

We can also compare our sample complexity bound to the result of \citet{liu2018simple} for pure random exploration. \citet{liu2018simple} bound the covering length $L$ and rely on the results of \citet{even2003learning} to turn that into a sample-complexity bound for Q-learning. Their sample-complexity bound scales as $\omega\left(\frac{L^3}{\epsilon^2(1 - \gamma)^2}\right)$ while our \pref{thm:sample_complexity} in combination with \pref{sec:coverlen} gives $\tilde O\left(\frac{ L H^2 d \ln \bar N_\cF}{\epsilon^2}\right)$. A loose translation with $H \approx 1 / (1 - \gamma)$ and $d \ln \bar N_\cF \lesssim S^2A^2 \leq L^2$ shows that our results are never worse in tabular MDPs while being much more general (e.g., apply to non-tabular MDPs and capture many other favorable cases).

\subsection{Lower Bound}
We next provide a lower bound that shows that an $S / \legap(\cF, \cF)^2$ dependency in tabular MDPs is unavoidable. This suggests that Bellman-Eluder dimension (or alternative notions of statistical complexity such as Bellman-rank, Eluder-dimension, decoupling coefficient, etc., which are all bounded by $SA$ for tabular MDPs) alone are not sufficient to capture the performance of myopic exploration RL algorithms such as $\epsilon$-greedy. 

\begin{restatable}[Sample Complexity Lower Bound]{theorem}{lowerb}
\label{thm:lower_bound}
Let $\bar c = {3\sqrt{3}}/{32}$. For any given horizon $H \in \bbN$, number of states $S \geq 8$, number of actions $A \geq 2$, exploration parameter $\epsilon \in (0, 1)$ and $v \in [1, \bar c(\epsilon / A)^{H/2}]$, there exists a tabular MDP $\cM = (\cX, \cA, H, P, R)$ with $|\cA| = A$ and $|\cX| = S$ and a function class $\cF$ such that:
\begin{enumerate}[label=\((\alph*)\), topsep=0pt, itemsep=0mm] 
    \item $\legap(\cF', \cF) = \min_{f \in \cF'} \legap(f, \cF) \in \left[v \bar c \sqrt{\frac{3\epsilon}{A}},  v\right]$ where  $\cF' \subset \cF$ denotes the set of all the value functions that are at least $1/16$ suboptimal, i.e., $V^{\star}(x_1) - V^{\pi^f}(x_1) > 1/16$ for any \(f \in \cF'\); 
    \item the expected number of episodes for which  \pref{alg:egreedy_genfun} with $\epsilon$-greedy exploration does not select an $1/16$-optimal function $f \in \cF'$ is 
    \begin{align*}
        \Omega\prn[\Big]{\frac{S}{\legap(\cF', \cF)^2}}.
    \end{align*}
\end{enumerate} 
\end{restatable}

\subsection{Regret Bound for $\epsilon$-Greedy RL}
Equipped with the sample complexity bound in \pref{thm:sample_complexity}, we can derive regret bounds for myopic exploration based RL. In the following, we derive regret bounds for $\epsilon$-greedy algorithm. Note that the regret can be decomposed as  
\begin{align*}
    &\regret(T) = \sum_{t = 1}^T \brk{V_1^\star(x_1) - V^{\wt \pi_t}_1(x_1)}\\
    &= \sum_{t = 1}^T \brk{V_1^\star(x_1) - V^{\pi_t}_1(x_1)} 
    + \sum_{t = 1}^T \brk{V_1^{\pi_t}(x_1) - V^{\wt \pi_t}_1(x_1)}
\end{align*}
%} 
where $\pi_t = \pi^{\wh f_t}$ denotes the greedy policy at episode \(t\). The second term in the above decomposition is bounded by $\epsilon H T$ since the return of greedy and exploration policy can differ at most by $\epsilon H$ in each episode.
The first term denotes the regret of the corresponding greedy policies, which can be  controlled by invoking \pref{thm:sample_complexity} to bound the number of episodes for which the greedy policies are suboptimal. For  favorable learning tasks in which every $f \in \cF$ with a suboptimal greedy policy has a significant myopic exploration gap, we can simply invoke \pref{thm:sample_complexity} with $\cF' = \cF^{\operatorname{sub}} = \{ f \in \cF \colon \pi^{f} \textrm{is not optimal}\}$. For illustration, consider the learning task in \pref{fig:effect_of_rewards} (middle) where gaps are large (cf. \pref{eqn:breadcrumb_large_gaps}). In this case,  
\begin{align*}
\regret(T) &= \epsilon H T + \widetilde O\prn[\Big]{\frac{H^2 d}{\alpha(\cF^{\operatorname{sub}}, \cF)^2}} \\
&= \epsilon H T +  \widetilde O\prn[\Big]{\frac{SA^{1 + B }H^4}{\epsilon^{B}}},
\end{align*}
Setting $\epsilon \approx A \left({SH^3}/{T}\right)^{{1}/(B+1)}$ thus  yields the bound 
\begin{align*}
 \regret(T) = \widetilde O\prn[\big]{HA T^{\frac{B}{B+1}} (SH^3)^{\frac{1}{B+1}}}~.
\end{align*}
Clearly, as shown by the above example, the optimal choice of $\epsilon$ (and thus the regret) depends on how the myopic exploration gap scales with $\epsilon$. We formalize this dependence in the following theorem.
\begin{restatable}[Regret Bound of $\epsilon$-Greedy]{theorem}{egreedyregret}
\label{thm:regret_egreedy}
Let $T \in \bbN$ and 
suppose we run \pref{alg:egreedy_genfun} with $\epsilon$-greedy exploration and a function class $\cF$ that satisfies \pref{ass:completeness}. 
Further, let there be a $h \in [1, H]$ such that for any $\lambda \in [0,1]$, 
$\legap(\cF'_{\lambda}, \cF) \geq \Omega((\epsilon / A)^{h/2} \lambda)$
where $\cF'_{\lambda} \subset \cF$ denotes the set of all the value functions that are at least $\lambda$ suboptimal. Then, with probability at least $1 - \delta$,  we have,
\begin{align*}
     \operatorname{Reg}(T) \leq \epsilon H T + \widetilde  O \prn[\Big]{
    \sqrt{\frac{h A^h d H^3 T}{\epsilon^h} \ln \frac{\bar N_{\cF}(T^{-1})}{\delta}}},
\end{align*}
where $d = \operatorname{dim}_{\operatorname{BE}}(\cF, \Pi_{\cF}, 1/\sqrt{T})$ denotes the Bellman-Eluder dimension of $\cF$. Furthermore, setting the exploration parameter $\epsilon = \widetilde \Theta\prn[\big]{\prn[\big]{\frac{h H A^h d}{T}}^{\frac{1}{2 + h}}}$, we get that 
\begin{align*}
   \operatorname{Reg}(T) \leq  \widetilde O\prn[\Big]{ H^{\frac{h+3}{h+2}} T^{\frac{h+1}{h+2}} \prn[\big]{h A^h d  \ln \frac{\bar N_{\cF}(T^{-1})}{\delta}}^{\frac{1}{h+2}}}. 
\end{align*}
\end{restatable}

For the contextual bandits problems, \pref{cor:contextual_bandits} implies that $\legap(\cF'_{\lambda}, \cF) \geq \Omega((\epsilon / A)^{1/2} \lambda)$ and thus \(h=1\). Plugging this in \pref{thm:regret_egreedy} gives us \(\operatorname{Reg}(T) = \widetilde{O}(H^{4/3} A^{1/3} T^{2/3})\), which matches the optimal regret bound for $\eps$-greedy for the contextual bandits problem in terms of dependence on \(T\) or \(A\) \cite{lattimore2020bandit}. In the worst case for any RL problem, we always have that $\legap(\cF'_{\lambda}, \cF) \geq \Omega((\epsilon / A)^{H/2} \lambda)$ and for such problems \pref{thm:regret_egreedy} still results in a sub-linear regret bound of \(\widetilde{O}(T^{(H+1)/(H+2)})\).

\subsection{Proof Sketch of \pref{thm:sample_complexity}} \label{sec:sample_complexity_sketch} 
We first partition $\cF'$ into subsets $\crl{\cF'_1, \dots \cF'_{\ln c(\cF', \cF)}}$ such that the myopic exploration radius $\lerad(f, \cF)$ is roughly the same for each \(f \in \cF'_i\). The proof follows by bounding the number of episodes $\cK_i = \{ t \in [T] \colon \wh f_t \in \cF'_i\}$ individually. In order do so, we consider the potential
\begin{equation}
    \sum_{t \in \cK_i} V^{\pi'_t}(x_1) - V^{\pi_t}(x_1),
    \label{eq:potential}
\end{equation}
where $\pi_t = \pi^{\wh f_t}$ and  $\pi'_t \in \Pi_{\cF}$ denotes the policy that attains the maximum in \pref{eq:local_gap} for $f = \wh f_t$. We will bound this potential from above and below as
\begin{align*}
    |\cK_i| \legap(\cF'_i, \cF) \sqrt{\lerad(\cF'_i, \cF)} \leq \eqref{eq:potential} 
    \leq \widetilde O\prn[\big]{H \sqrt{c(\cF'_i, \cF) |\cK_i| d}}
\end{align*}
which implies that $|\cK_i| = \widetilde O\prn{{ H^2 d}/{\legap(\cF_i', \cF)^2}}$,  yielding the desired statement after summing over $i$. The lower bound follows immediately from the definition of $\legap$ and $\lerad$. For the upper bound, we first compose each value difference into expected Bellman errors using
\begin{restatable}{lemma}{regretdecomp}
\label{lem:regret_decomp} 
Let $f = \{f_h\}_{h \in [H]}$ with $f_h \colon \cX \times \cA \rightarrow \bbR$ and $\pi^f$ be the greedy policy of $f$. Then for any policy $\pi' \in \Pi$,
\begin{align*}
    &V^{\pi'}_1(x_{1}) - V^{\pi^f}_1(x_{1}) 
    \leq \\
    &\sum_{h=1}^H \bbE_{\pi^f}[(\cE_{h}f)(x_{h}, a_{h})]
    - \sum_{h=1}^H \bbE_{\pi'}[(\cE_{h}f)(x_{h}, a_{h})].
\end{align*}
\end{restatable}
This yields an upper-bound on \pref{eq:potential} of 
\begin{align*}
    \sum_{h=1}^H \brk[\Big]{
    \underset{(A)}{\underbrace{\sum_{t \in \cK_i}\! |\bbE_{\pi'_t}[(\cE_{h} \wh f_t)(x_{h}, a_{h})]|}}
    + \!\!\underset{(B)}{\underbrace{\sum_{t \in \cK_i} \!|\bbE_{\pi_t}[(\cE_{h} \wh f_t)(x_{h}, a_{h})]|}}\!}
\end{align*}
Now, both the terms $(A)$ and $(B)$ can be bound using the standard properties of  Bellman-Eluder dimension as:
\begin{align*}
\scalebox{0.9}{$\displaystyle
    \max\crl{(B),  (A)} \lesssim \sqrt{d |\cK_i| \max_{t \in \cK_i} \sum_{\tau \in \cK_i \cap [t-1]} \prn[]{\bbE_{\pi_\tau}[(\cE_{h} \wh f_t)(x_{h}, a_{h})]}^2}.$}
\end{align*}
The remaining sum can be controlled using Jensen's inequality and the definition of myopic exploration gap as
\begin{align*}
    \sum_{\tau \in \cK_i \cap [t-1]} \!\!\!\!\!\!\!\!\bbE_{\pi_\tau}[(\cE_{h} \wh f_t)&(x_{h}, a_{h})]^2
    \leq \!\!\!\!\!
    \sum_{\tau \in \cK_i \cap [t-1]} \!\!\!\!\!\!\!\!\bbE_{\pi_\tau}[(\cE_{h}^2 \wh f_t)(x_{h}, a_{h})]\\
    \leq &
    c(\cF_i', \cF) \sum_{\tau \in \cK_i \cap [t-1]}  \bbE_{\wt \pi_\tau}[(\cE_{h}^2 \wh f_t)(x_{h}, a_{h})]\\
    \lesssim  & c(\cF_i', \cF) \ln \frac{\bar N_{\cF}(1/T) \ln(t)}{\delta},
\end{align*}
where the last inequality follows from a uniform concentration bound for square loss minimization. 

\section{Related Work}

The closest work to ours is \citet{liu2018simple} who provide conditions under which uniform exploration yields polynomial sample-complexity in infinite-horizon tabular MDPs, building on the Q-learning analysis of \citet{even2003learning} (see \pref{sec:coverlen}). Many conditions that enable efficient myopic exploration also allow us to use a smaller horizon for planning in the MDP, a question studied by \citep{jiang2016structural} (e.g., small action variation). However, sufficient conditions for shallow planning are in general not sufficient for efficient myopic exploration and vice versa.
One can for example construct cases where planning with horizon $H' = 1$ yields the optimal policy (i.e., $\pi^\star(x_h) = \arg\max_{a} r(x_h,a)$) for the original horizon $H$ but there are distracting rewards just beyond the shallow planning horizon $H' + 1$ that would throw off algorithms with $\epsilon$-greedy (see also the example in \pref{app:myopic_exploration_further}).

\citet{simchowitz2020naive} and \citet{mania2019certainty} show that simple random noise explores optimally in linear quadratic regulators, but their analysis is tailored specifically to this setting. 

There is a rich line of work on understanding the effect of reward functions and on designing suitable rewards, to speed up the rate of convergence of various RL algorithms and to make them more interpretable \cite{abel2021expressivity, devidze2021explicable, hu2020learning, mataric1994reward, icarte2022reward, marthi2007automatic}. While being extremely interesting, the problem of designing suitable reward functions to model the underlying objective is orthogonal to our focus in this paper, which is to understand rate of convergence for myopic exploration algorithms. 

Potential based rewards shaping is a popular approach in practice \cite{ng1999policy} to speed up the rate of convergence of RL algorithms. In order to quantify the role of reward shaping,   \citet{laud2003influence} provide an algorithm for which they demonstrate, both theoretically and empirically, improvement in the rate of convergence from reward shaping.  
However, their algorithm is qualitatively very different from the myopic exploration style algorithms that we consider in this paper, and is in general not efficiently implementable for MDPs with large state spaces. For a more detailed discussion of potential-based reward shaping, see \pref{sec:reward_shaping} and \pref{app:myopic_exploration_further}.

\pref{alg:egreedy_genfun} determines the Q-function estimate by a finite-horizon version of fitted Q-iteration \citep{ernst2005tree}. In the infinite horizon setting, this procedure has been analyzed by \citet{munos2008finite, antos2007fitted}. These works focus on characterizing error propagation of sampling and approximation error and simply assume sampling from a generative model or fixed behavior policy that explores sufficiently, i.e., has good state-action coverage. A recent line of work on offline RL aims to relax the coverage and additional structural assumptions, see e.g. \citet{uehara2021pessimistic} and references therein for an overview. Our work instead has a different focus. We avoid approximation errors by \pref{ass:completeness} but aim to characterize the interplay between MDP structure and the different exploration policies and their impact on the efficiency of online RL.

\section{Conclusion and Future Work}
\label{sec:conclusion}

We view this work as a first step towards fine-grained theoretical guarantees for practical algorithms with myopic exploration. We provided a complexity measure that captures many favorable cases where such algorithms work well. We focused on general results that apply to problems with general function approximation, due to the importance of myopic exploration in such settings.
One important direction for future work is an analysis of myopic exploration specialized to tabular MDPs. We believe that a finer characterization of the performance of $\epsilon$-greedy algorithms in this setting is possible by making explicit assumptions on the initialization of  function values in state-action pairs that have not been visited so far. We purposefully avoided such assumptions in our work to avoid conflating myopic exploration mechanisms with those of optimistic initializations which are known to be effective. 
It would further be interesting to compare the sample complexity of different myopic exploration strategies such as softmax-policies, additive noise or $\epsilon$-greedy and perhaps develop new myopic strategies with improved bounds. 

\subsection*{Acknowledgements}

YM received funding from the European Research Council (ERC) under the European Union’s Horizon 2020 research and innovation program (grant agreement No. 882396), the Israel Science Foundation (grant number 993/17), Tel Aviv University Center for AI and Data Science (TAD), and the Yandex Initiative for Machine Learning at Tel Aviv University. KS acknowledges support from NSF CAREER Award 1750575. 

\clearpage 
\bibliography{refs} 
\bibliographystyle{icml2022}

\clearpage

\appendix
\onecolumn

\section{Additional Discussion of Myopic Exploration Gap}
\label{app:myopic_exploration_further} 
To provide further intuition behind the definition of myopic exploration gap in \pref{def:local_exploration_gap}, we discuss the following simple example. For any horizon $H \in \bbN$, consider an MDP with states $\cS = [2^H - 1]$ organized in a binary tree. The agent starts at the root of the tree and each action in $\cA = \{0, 1\}$ chooses one of the two branches to descent. There is no stochasticity  in the transitions and the agent always ends up on one of the leaves of the tree in the final time step of the episode. The function class $\cF$ corresponds to the set of all value functions for this tabular MDP. 
The goal of the agent is to reach a certain leaf state $s_\star$. We consider three different reward functions to formulate this objective (we provide an illustration of the three rewards for \(H = 3\) in \pref{fig:illustration_tree}): 
\paragraph{Goal reward:} The agent only receives a reward when it reaches the state $s_\star$, i.e., 
\begin{align*} 
    r(s, a) = \indicator{s = s_\star}
\end{align*}
Let $s_1', a_1', s_2', a_2' \dots, s_{H-1}', a_{H-1}', s_\star$ be the unique path that leads to $s_\star$, and let $f \in \cF$ be any Q-function so that $\pi_f(s_h') \neq a_h'$ for all \(h \in [H]\). The myopic exploration gap of this function for $\epsilon$-greedy with sufficiently small $\epsilon$ is 
\begin{align*}
    \legap(f, \cF) = \left(\frac{\epsilon}{2}\right)^{\frac{H-1}{2}}.
\end{align*}
This is true because only $\pi_\star$ achieves higher return than $\pi_f$ but while $\pi_\star$ reaches $s_\star$ with probability $1$, $\myopic(f)$ only does so with probability $(\epsilon/2)^{H-1}$ and, hence, $c = (\epsilon/2)^{H-1}$. 

\paragraph{Path reward:} The agent receives a positive reward for any right action along the optimal path, i.e., 
\begin{align*}
    r'(s,a) = \frac{1}{H}\indicator{\exists i \in [H] \colon (s,a) = (s_i', a_i')}.
\end{align*}
Here, the myopic exploration gap of any $f \in \cF$ with suboptimal greedy policy $\pi_f$ is
\begin{align*}
    \legap'(f, \cF) = \frac{\sqrt{\epsilon}}{H\sqrt{2}}
\end{align*}
because there is another $f' \in \cF$ which is identical to $f \in \cF$ except that the values of $s_h'$, the last state on the desired path taken by $\pi_f$, are so that $\pi_{f'}(s_h') = a_h'$, i.e., $\pi_{f'}$ stays at least one time step longer on the optimal path. As we can see, the myopic exploration gap for this reward formulation is much more favorable. This is similar to the breadcrumb example in \pref{fig:effect_of_rewards}. Indeed, $\epsilon$-greedy exploration is much more effective in this formulation.

\paragraph{Potential-based shaping of goal reward:} Potential-based reward shaping is a popular technique for designing rewards that may be beneficial to improve the speed of learning, while retaining the policy preferences of given reward function \citep{ng1999policy}. We here consider a reward function that give a reward of $+1$ if the agent takes the first right action and then a $-1$ whenever it takes a wrong action afterwards. Formally, this reward is
\begin{align*}
    r''(s,a) = \begin{cases}
        +1 & \textrm{if } s = s_1' \textrm{ and } a = a_1'\\
        -1 & \textrm{if } s = s_h' \textrm{ for $1 < h < H$ and } a \neq a_h'\\
        0 & \textrm{otherwise}
    \end{cases}
\end{align*}

\begin{figure}[htp]

\centering
\includegraphics[width=.60\textwidth]{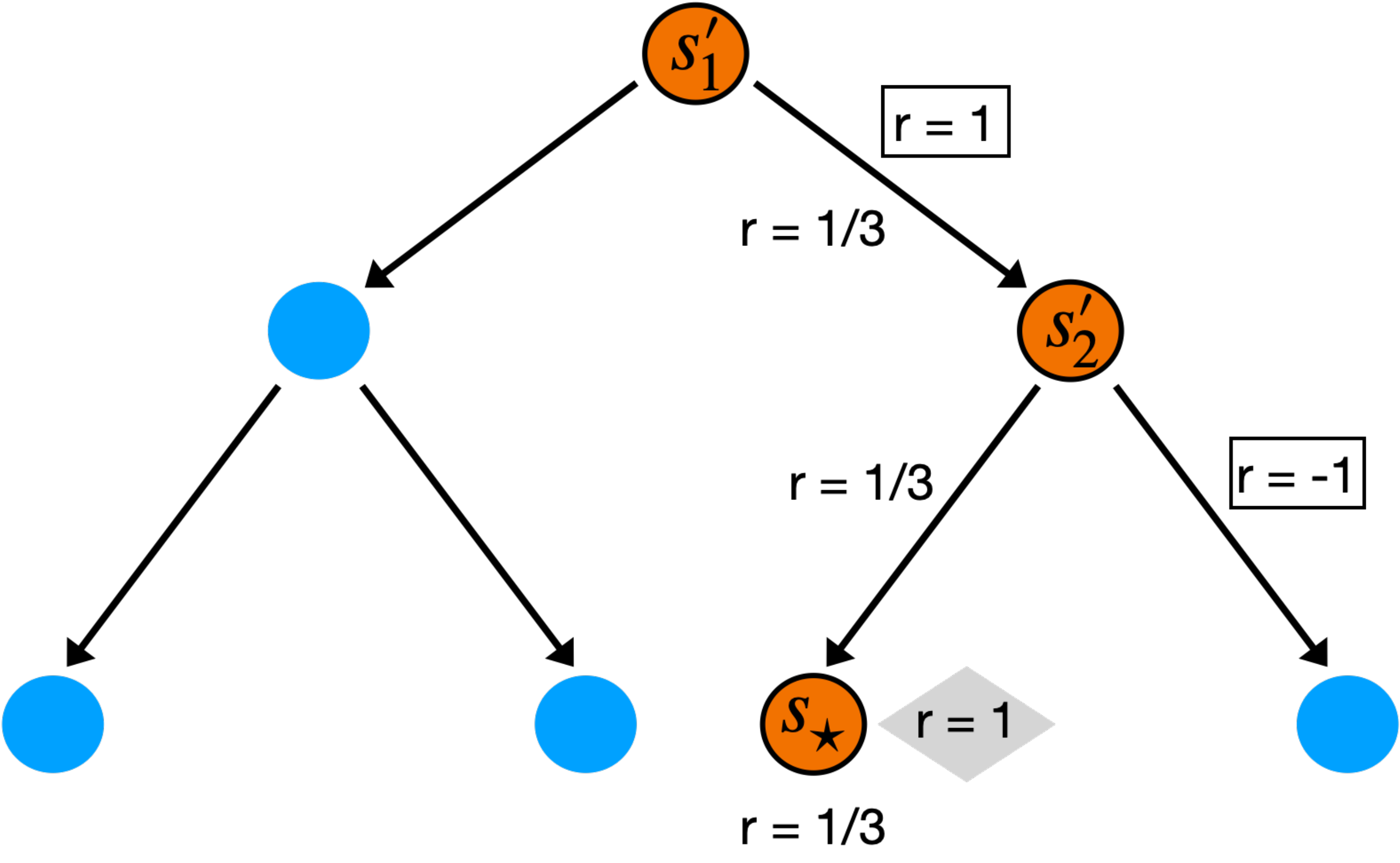}\hfill
\caption{An illustration of different reward function for \(H = 3\). Orange colored circles (circles that have a boundary) denote the optimal path \(s_1', s_2', s_\star\). The \textit{goal rewards} are specified in the shaded diamond. The \textit{path rewards} are specified without any boundary. The \textit{potential based shaped rewards} are specified in rectangular boxes. Corresponding to each reward type (given by shapes: diamond, none, boxes) any unspecified edges denote a reward of \(0\).}
\label{fig:illustration_tree}

\end{figure}

The potential function $\Phi \colon \cS \rightarrow \bbR$ that transforms $r$ into $r''$ is $\Phi(s) = - \indicator{\exists h \in [2, H] \colon s = s'_h}$. One can verify easily that $r''(s,a) = r(s,a) + \Phi(s) - \bbE_{s' \sim P(s,a)}[\Phi(s')]$ and that the return of any policy under $r$ and $r''$ is identical. Note that $r''$ is not in the range $[0,1]$. One may apply a linear transformation to all reward functions to normalize $r''$ in the range $[0, 1]$, however in the following, we work with ranges $[-1, 1]$ for reward and value functions for the ease of exposition since this does not change the argument. Since we are in the tabular setting, we can use \pref{eqn:local_gap_tabular} to compute the myopic exploration gap. Since the return of each policy under $r$ and $r''$ is identical, all myopic exploration gaps under $r$ and $r''$ are identical. This example illustrates the following general fact:
\begin{fact}
Myopic exploration gaps are not affected by potential-based reward shaping in tabular Markov decision processes.
\end{fact}
Here, this implies that the myopic exploration gap under $r''$ of any Q-function $f \in \cF$  with $\pi_f(s_h') \neq a_h'$ for all $h \in [H]$ is  
\begin{align*}
    \legap''(f, \cF) = \left(\frac{\epsilon}{2}\right)^{\frac{H-1}{2}}.
\end{align*}
This suggests that the sample complexity of $\epsilon$-greedy exploration in this example is exponential in $H$. This may be surprising since the myopic greedy policy that maximizes only the immediate reward, i.e. takes actions $\argmax_{a \in \cA} r''(s, a)$, is optimal in this problem. This implies that shallow planning with horizon $1$ is sufficient in this problem and may suggest that the sample complexity is \emph{not} exponential in $H$. How can this conundrum be resolved and what is the sample-complexity of $\epsilon$-greedy in this problem?

We will argue in the following that the efficiency of $\epsilon$-greedy in this problem depends how we initialize the Q-function estimate of state-action pairs that have not been visited. There are initializations under which the sample-complexity is polynomial or exponential in $H$ respectively. We therefore conclude that the benefit of potential-based reward shaping in this case is not due to exploration through $\epsilon$-greedy but rather optimism or pessimism in the initializations. Note that our procedure in \pref{alg:egreedy_genfun} does not prescribe an initialization (all initializations are minimizers of $\cL_h(f_{t, h+1})$).

First consider initializing the Q-function table with all entries to be equal to $0$. In this case, the agent will try action $a_1'$ in state $s_1'$ after at most $\Omega(1 / \epsilon)$ episodes. Independent of which actions it took afterwards, the V-value estimate for $s_2'$ will be $0$ and will always remain this value since everything is deterministic in this problem and $0$ is the correct value. As a result the Q-function estimate for $(s_1', a_1')$ is $+1$ and the greedy policy will take $a_1'$ in $s_1'$. Now, in any episode where the agent actually follows the greedy policy, it learns that one action that deviates from the optimal path is suboptimal. Hence, after $\text{O}(H/\epsilon)$ episodes, the algorithm will choose the optimal policy as its greedy policy and even learn the optimal value function. Hence, the sample-complexity of $\epsilon$-greedy with this intialization is indeed polynomial in $H$. Interestingly, the $\epsilon$-greedy with the same $0$-initialization has exponential sample complexity for the original goal-based rewards $r$. This is because none of the Q-function estimates changes until the agent first reaches $s_\star$. Essentially, the same initialization is pessimistic under the original reward function $r$ but optimistic under the shaped version $r''$.

Now consider initializing the Q-function table with all entries to be equal to $-1$. The agent will try action $a_1'$ in state $s_1'$ after at most $\Omega(1 / \epsilon)$ episodes. Unless it happens to exactly follow the optimal path, which only happens with probability $(1/2)^{H-1}$, the agent will learn to associate a Q-value of $0$ for the initial action and a Q-value of $-1$ for all actions along the optimal path afterwards. Hence, it has no preference between the actions in any state of the optimal path (except for the first action) and would still need at least $(1/2)^{H-2}$ episodes to randomly follow the optimal path and discover the reward of $0$ in the final state. Hence, the sample-complexity of $\epsilon$-greedy with this initialization is indeed exponential in $H$, since there was no optimism in the initialization and $\epsilon$-greedy was ineffective at exploring for this problem.

\section{Proofs For Bounds on Myopic Exploration Gap in \pref{sec:local_exploration_gap}} 

\legapgeneral* 
\begin{proof} 
For the upper-bound, note that the objective of \pref{eq:local_gap} is bounded for all $c \geq 1$ and $\pi' \in \Pi_{\cF} \subseteq \Pi$, as
\begin{align*}
\frac{1}{\sqrt{c}} (V_1^{\pi'}(x_1) - V_1^{\pi^f}(x_1)) &\leq V_1^{\pi'}(x_1) - V_1^{\pi^f}(x_1) \\ &\leq V_1^{\star}(x_1) - V_1^{\pi^f}(x_1).
\end{align*}
 For the lower-bound, we show that $\pi' = \pi^\star$ and $c = \max_{\pi' \in \Pi'} \left\| \frac{\bbP_{\pi'}}{\bbP_{\myopic(f)}}\right\|_{\infty}$ is a feasible solution for \pref{eq:local_gap}. Since $Q^\star \in \cF$ by realizability, any optimal policy $\pi^\star \in \Pi_{\cF} = \Pi'$ is feasible as the expected bellman error for \(f'\) corresponding to \(\pi^*\) is equal to \(0\). Further, for any $\pi' \in \Pi'$ and function $g \colon \cX \times \cA \rightarrow \bbR^{+}$ we have
\begin{align*}
    \bbE_{\pi'}[g(x_h, a_h)]
    \leq \bbE_{\myopic(f)}[g(x_h, a_h)]\left\| \frac{\bbP_{\pi'}}{\bbP_{\myopic(f)}}\right\|_{\infty}
\end{align*}
which shows that the value of $c$ is feasible.
\end{proof}

\coregreedygen*
\begin{proof}
The likelihood ratio of an episode $\tau = (x_1, a_1, r_1, \dots, x_{H}, a_{H}, r_{H}, x_{H+1})$ w.r.t.~$\myopic(f)$ and any other policy $\bar \pi \in \Pi$ satisfies
\begin{align*}
    \frac{\bbP_{\bar \pi}(\tau)}{\bbP_{\myopic(f)}(\tau)} \leq \prod_{h=1}^H \frac{\bar \pi(a_h | x_h)}{\myopic(f)(a_h| x_h)}
\leq \prod_{h=1}^H \frac{1}{\epsilon / A} = \left(\frac{A}{\epsilon}\right)^H\!\!\!.
\end{align*}
The result then follows from \pref{lem:legap_general_bound}.
\end{proof}

\multiactionvar*
\begin{proof} First note that for any policy \(\pi\), the occupancy measure for state \(x_h\) and action \(a_h\) is given by  
\begin{align*}
\mu^{\pi}_h(x_h, a_h) &= \sum_{x_1, \dots, x_{h-1}} \prn[\Big]{\prod_{j=1}^{h-1} \sum_{a \in \cA} \Pr(\pi(x_j) = a) P_j(x_{j+1} \mid x_j, a)} \Pr(\pi(x_h) = a_h),  \numberthis \label{eq:mu_definition} 
\end{align*} where \(P\) denotes the transition dynamics corresponding to \(\cM\) and we used the fact that \(x_1\) is fixed. Next, note that as a consequence of \pref{def:multiplicative_action_variation}, we have that for any \(x_j, x_{j+1}, a\) and \(a'\), 
\begin{align*}
P_j(x_{j+1} \mid x_j, a) \leq \delta_P P_j(x_{j+1} \mid x_j, a'). 
\end{align*}
Thus, we have that: 
\begin{align*}
P_j(x_{j+1} \mid x_j, a) &\leq \sum_{a' \in \cA} \Pr(\myopic(f)(x_j) = a')  P_j(x_{j+1} \mid x_j, a) \\ &\leq \delta_P \prn[\Big]{\sum_{a' \in \cA} \Pr(\myopic(f)(x_j) = a') P_j(x_{j+1} \mid x_j, a') }. 
\end{align*}
Plugging the above in \pref{eq:mu_definition}, we get that: 
\begin{align*}
\mu^{\pi}_h(x_h, a_h) &\leq \sum_{x_1, \dots, x_{h-1}} \prn[\Big]{\prod_{j=1}^{h-1} \sum_{a \in \cA} \Pr(\pi(x_j) = a)  \delta_P \prn[\Big]{\sum_{a' \in \cA} \Pr(\myopic(f)(x_j) = a') P_j(x_{j+1} \mid x_j, a') }} \Pr(\pi(x_1) = a) \\ 
&=  \sum_{x_1, \dots, x_{h-1}} \prn[\Big]{\prod_{j=1}^{h-1}  \delta_P \prn[\Big]{\sum_{a' \in \cA} \Pr(\myopic(f)(x_j) = a') P_j(x_{j+1} \mid x_j, a') }} \Pr(\pi(x_h) = a_h) \\  
&= \delta_P^{h-1} \sum_{x_1, \dots, x_{h-1}} \prn[\Big]{\prod_{j=1}^{h-1} {\sum_{a' \in \cA} \Pr(\myopic(f)(x_j) = a') P_j(x_{j+1} \mid x_j, a') }} \Pr(\pi(x_h) = a_h).
\intertext{Next, note that \(\myopic(f)\) is the \(\epsilon\)-greedy policy and thus \({\Pr(\myopic(x_h) = a_h)} \geq {\epsilon} / A\). Using this fact in the above bound, we get that}
\mu^{\pi}_h(x_h, a_h) &\leq \delta_P^{h-1} \sum_{x_1, \dots, x_{h-1}} \prn[\Big]{\prod_{j=1}^{h-1} {\sum_{a' \in \cA} \Pr(\myopic(f)(x_j) = a') P_j(x_{j+1} \mid x_j, a') }} \frac{A \Pr(\myopic(x_h) = a_h)}{\epsilon} \\
&= \frac{A \delta_P^{h-1}}{\epsilon} \sum_{x_1, \dots, x_{h-1}} \prn[\Big]{\prod_{j=1}^{h-1} {\sum_{a' \in \cA} \Pr(\myopic(f)(x_j) = a') P_j(x_{j+1} \mid x_j, a') }} \Pr(\myopic(x_h) = a_h) \\ 
&\leq  \frac{A \delta_P^H}{\epsilon} \mu^{\myopic(f)}_h(x_h, a_h),  \numberthis \label{eq:mu_definition_2} 
\end{align*} where the inequality in the last line holds because \(\delta_P \geq 1\) and by using the definition of $\mu^{\myopic(f)}_h(x_h, a_h)$ from \pref{eq:mu_definition}.  

Observe that \pref{eq:mu_definition_2} holds for any policy \(\pi\). Thus, using this relation for \(\pi^\star\), we get that 
\begin{align*}
    \bbE_{\pi^\star}[(\cE_h^2 f')(x_h, a_h)] &= \sum_{x_h, a_h} \mu_h^{\pi^\star}(x_h, a_h) \cdot (\cE_h^2 f')(x_h, a_h) \\ &\leq  \frac{A \delta_P^H}{\epsilon} \mu_h^{\myopic(f)}(x_h, a_h) \cdot (\cE_h^2 f')(x_h, a_h) \\
    &\leq \frac{A \delta_P^H}{\epsilon} \bbE_{\myopic(f)}[(\cE_h^2 f')(x_h, a_h)]. \numberthis \label{eq:mu_definition3} 
\end{align*} 

A similar analysis reveals that 
\begin{align*}
\bbE_{\pi^f}[(\cE_h^2 f')(x_h, a_h)] &\leq \frac{A \delta_P^H}{\epsilon} \bbE_{\myopic(f)}[(\cE_h^2 f')(x_h, a_h)]. \numberthis \label{eq:mu_definition4} 
\end{align*}
    
The relations in \pref{eq:mu_definition3} and \pref{eq:mu_definition4} thus imply that \(\pi^\star \in \Pi'\) satisfies the constraints in the definition of \(\alpha(f, \cF)\) with $c = \frac{A \delta_P^H}{\epsilon}$. We thus have that  
\begin{align*}
\alpha(f, \cF) &\geq \frac{1}{\sqrt{c}} (V_1^{\pi^\star}(x_1) - V_1^{\pi^f}(x_1)) = \sqrt{\frac{\epsilon}{A \delta_P^H}} \cdot (V_1^{\pi^\star}(x_1) - V_1^{\pi^f}(x_1)).
\end{align*}
\end{proof}

\section{Proof of Sample Complexity Lower Bound} 
\label{app:lower_bound_proofs}

\begin{figure}
\vspace{3mm}
    \centering
    \includegraphics[width=0.8\linewidth]{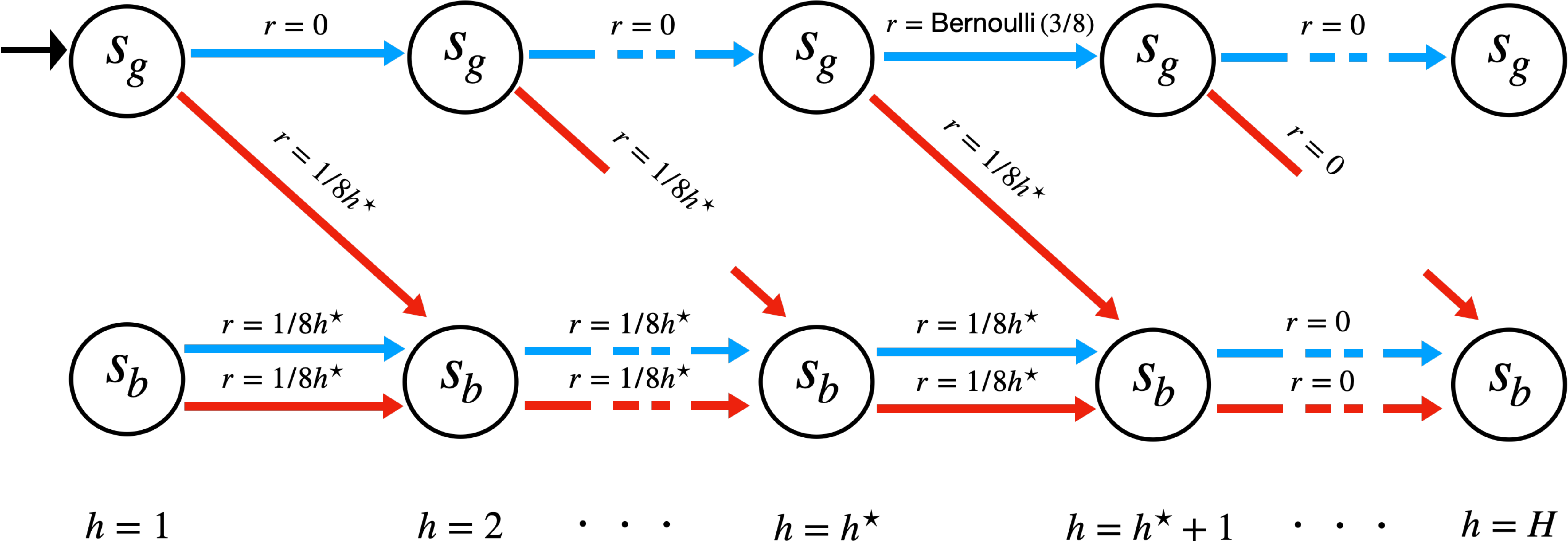}
    \caption{MDP constructon for the proof of \pref{thm:lower_bound}. There are two states \(\crl{s_g, s_b}\) and \(A\) action. The black arrow denotes the starting state. The blue arrow denotes the deterministic transition when the agent takes action \(a = 1\), and the red arrow denotes deterministic transition when the agent takes any other action in \([A] \setminus \crl{1}\). As long as the agent picks actions $1$, it stays on the good chain $s_{g}$ but as soon as it chooses any other actions it transitions to the bad chain $s_{b}$ where it will stay forever. The observed rewards are time dependent and are shown over the corresponding action arrows.}
    \label{fig:lower_bound_dynamics}
\end{figure}

\lowerb*
\begin{proof}
We first will show the statement for $S=2$ and then extend the proof to $S > 2$. 

\paragraph{Construction for $S = 2$.}
Let $A \in \bbN$, $H \in \bbN$ and $\epsilon$ be fixed and consider
states $\cX = \{s_{g}, s_{b}\}_{h  \in [H+1]}$ and actions $\cA = [A]$. 
The agent always starts in state $x_{\textrm{init}} = s_{g}$. The dynamics is deterministic, non-stationary and defined as
\begin{align*}
    P_h(s_{g} | s_{g}, a = 1) &= 1 &
    P_h(s_{b} | s_{g}, a \neq 1) &= 1 &
    P_h(s_{b} | s_{b}, a \in \cA) &= 1
\end{align*}
for $h \in [H]$. All other transitions have probability $0$. Essentially, the state space forms two chains, and the agent always progresses on a chain. As long as the agent picks actions $1$, it stays on the good chain $s_{g}$ but as soon as it chooses any other actions it transitions to the bad chain $s_{b}$ where it will stay forever. The function class $\cF = \{\cF_h\}_{h \in [H]}$ is the full tabular class, i.e., $\cF_h = \cX \times \cA \rightarrow [0, H + 1 - h]$ for $h \in [H]$.
Now, for the given value $v$, we define the reward distributions $R = \{R_h\}_{h \in [H]}$ as 
\begin{align*}
    R_h(x, a) &= 
    \begin{cases}
    0 & \textrm{if } h > h^\star\\     \operatorname{Bernoulli}\left(\frac{3}{8}\right)
    & \textrm{if }h = h^\star \textrm{ and }x = s_{g} \textrm{ and }a = 1\\
    0 & \textrm{if }h < h^\star \textrm{ and }x = s_{g} \textrm{ and }a = 1\\
    \frac{1}{8h^\star} & \textrm{otherwise} 
    \end{cases} 
\end{align*} 
where $h^\star = \lceil 2\log_{\epsilon / A} (4v) \rceil \leq H$. We illustrate the transition dynamics and the corresponding reward function in \pref{fig:lower_bound_dynamics}. 
    
\paragraph{Value of myopic exploration gap.}
We now show that the smallest non-zero exploration gap is $v$ up to a constant factor in this MDP. The value of a deterministic policy is only determined by how long it stays in $s_{g}$ (by always taking action $a = 1$). For any $\pi \in \Pi_{\det}$, let $L(\pi) \in [H]$ denote that length. We have
\begin{align*}
    V^\pi(x_1) = \frac{3}{8}\indicator{L(\pi) \geq h^\star} + \frac{1}{8}\left(1 - \frac{L(\pi)}{h^\star}\right)\indicator{L(\pi) < h^\star} 
\end{align*}
The value of $\alpha(f, \cF, \Pi_{\cF}, \myopic, \cM)$ can be lower-bounded for any $f \in \cF$ with $L(\pi^f) < h^\star$ by considering $\pi' = \pi^\star$ in \pref{eq:local_gap}. This gives
\begin{align*}
    \alpha(f, \cF, \Pi_{\cF}, \myopic, \cM) &\geq \sqrt{\bbP_{\myopic(f)}(a_1 = 1, \dots, a_{h^\star} = 1)} \left( \frac{3}{8} - \frac{1}{8}\left(1 - \frac{L(\pi^f)}{h^\star}\right)\right)\\
    & = \left ( \frac{\epsilon}{A} \right)^{\frac{h^\star - L(\pi^f)}{2}} \left( \frac{1}{4} + \frac{1}{8}\frac{L(\pi^f)}{h^\star}\right)\\
    & \geq \frac{1}{4} \left ( \frac{\epsilon}{A} \right)^{h^\star / 2}.
\end{align*}
Further, both inequalities are tight for any $f \in \cF$ with $f_h(s_g, 1) < \max_{a \in \cA \setminus \{1\}} f_h(s_g, a)$ for all $h \leq h^\star$, i.e., value functions for which the greedy policy would always choose to go to $s_b$ in the first $h^\star$ steps.
We therefore have shown that 
\begin{align*}
    \min_{f \in \cF \colon \alpha(f, \cF, \Pi_{\cF}, \myopic, \cM) > 0} \alpha(f, \cF, \Pi_{\cF}, \myopic, \cM)
    &= \frac{1}{4} \left ( \frac{\epsilon}{A} \right)^{h^\star / 2}
    \in \left[v \frac{1}{4}\sqrt{\frac{\epsilon}{A}},  v\right]~.
\end{align*}
\paragraph{Performance of $\epsilon$-greedy.}
Since the regression loss in \pref{alg:egreedy_genfun} accesses function values $\wh f_{h+1}(x', a')$ at state-action pairs $(x', a')$ for which the algorithm has never chosen $a'$ in $x'$, the behavior of $\epsilon$-greedy depends on their default value. 
To avoid a bias towards optimistic or pessimistic intialization, we assume that the datasets $\cD_h$ in \pref{alg:egreedy_genfun} are initialized with one sample transition $(x, a, r, x')$ from each $(x,a) \in \cX \times \cA$. An alternative to this assumption is to simply define the value function class $\cF$ given to the agent to be restricted to only those functions that match the optimal value as soon as an action $a \neq 1$ was taken, i.e.,
\begin{align*}
    \cF_h = \{ f \colon \cX \times \cA \rightarrow [0, 1] \mid f_h(x, a) = Q^\star_h(x, a) \forall (x, a) \textrm{ with } x = s_b \textrm{ or } a \neq 1\}
\end{align*}
In this case, the argument below applies as soon as the agent visits $(s_g, 1)$ at time $h^{\star}$ for the first time.

Note that the MDP $\cM$ is deterministic with the exception of the reward at $(s_g, 1)$ at time $h^\star$. Therefore, only two possible intializations are possible. With probability at least $1/4$, the algorithm was initialized with $(s_g, 1, 0, s_b)$ at time $h^\star$. In this case, we have
\begin{align*}
    \wh f_h(s_g, a) = 
    \begin{cases}
    \frac{h^\star - h}{8h^\star} & \textrm{if } a = 1\\
    \frac{h^\star - h + 1}{8h^\star} & \textrm{if } a \neq 1\\
    \end{cases}
\end{align*}
for all $h \leq h^\star$ and $\pi^{\wh f}$ would always choose a wrong action. Unless the agent receives a new sample from state-action pair $(s_g, 1)$ at time $h^\star$, this estimate will also not change since all other observations are deterministic. The probability with which $\myopic(\wh f)$ will receive such a sample in an episode is $(\epsilon / A)^{h^\star}$ and thus, the agent will require $\Omega(1/v^2)$ samples in expectation before it can switch to a different function.

\paragraph{Extension to $S \geq 8$.}
Without loss of generality, we can assume that $S$ is even, otherwise just choose $S \rightarrow S - 1$. We then create $n = S / 2$ copies of the 2-state MDP described above. The initial state distribution is uniform over all copies of $s_g$.\footnote{If we desire a deterministic start state, we can just increase the horizon $H \gets H + 1$ by one and have all actions transition uniformly to all copies in $h = 1$.}. The value of any deterministic policy is still only determined by how long it stays in $s_{g, i}$, each copy of $s_g$,
\begin{align*}
    \bbE_\pi[V^\pi(x_1)] = \frac{3}{8} \cdot\frac{1}{n} \sum_{i=1}^{n}\indicator{L_i(\pi) \geq h^\star} + \frac{1}{8}\cdot\frac{1}{n} \sum_{i=1}^{n}\left(1 - \frac{L(\pi)}{h^\star}\right)\indicator{L(\pi) < h^\star} 
\end{align*}
where $L_i(\pi)$ is the number of time steps the agent stays in $s_{g,i}$ for each copy $i$. The instantaneous regret of $\pi$ is then
\begin{align*}
    \bbE_\pi[V^\star(x_1) - V^\pi(x_1)] 
    &= 
    \frac{3}{8} -  \frac{1}{8}\cdot\frac{1}{n} \sum_{i=1}^{n}\left(1 + \frac{L(\pi)}{h^\star}\right)\indicator{L(\pi) < h^\star} 
    \\&=
    \frac{1}{n} \sum_{i=1}^{n}\left(\frac{1}{4} + \frac{L(\pi)}{8h^\star}\right)\indicator{L(\pi) < h^\star}
    \\&\geq \frac{1}{4} \cdot
    \frac{1}{n} \sum_{i=1}^{n} \indicator{L(\pi) < h^\star}.
\end{align*}
Thus, any policy $\pi$ with instantaneous regret at least $1/16$ needs to behave suboptimally in at least $3n/4$ of the $n$ copies. 
Let $\cF' = \{ f \in \cF \colon \bbE[V^\star(x_1) - V^{\pi^f}(x_1)] > 1/16 \}$ be all functions that have instantaneous regret at least $1/16$.
We can lower bound the myopic exploration gap for any $f \in \cF'$ by considering $\pi' = \pi^\star$ in \pref{eq:local_gap}. This gives
\begin{align*}
    \alpha(f, \cF, \Pi_{\cF}, \myopic, \cM) &\geq \sqrt{\bbP_{\myopic(f)}(a_1 = 1, \dots, a_{h^\star} = 1)} \frac{1}{4n} \sum_{i=1}^{n}\indicator{L_i(\pi^f) < h^\star} \\
    &\geq \sqrt{\bbP_{\myopic(f)}(a_1 = 1, \dots, a_{h^\star} = 1)} \frac{3}{16} \\
    &\geq \sqrt{\frac{3n}{4n} \left ( \frac{\epsilon}{A} \right)^{h^\star }} \frac{3}{16} 
    = \frac{3 \sqrt{3}}{32} \left(\frac{\epsilon}{A} \right)^{h^\star /2}
\end{align*}
Further, for $f \in \cF$ that behave optimally in $1/4n$ copies and choose $a_1 \neq 1$ in all other copies, all inequalities are tight. Hence,
\begin{align*}
    \min_{f \in \cF'} \alpha(f, \cF, \Pi_{\cF}, \myopic, \cM)
    &= \frac{3 \sqrt{3}}{32} \left(\frac{\epsilon}{A} \right)^{h^\star /2}
    \in \left[v\frac{3}{32} \sqrt{\frac{3\epsilon}{A}},  v\right],
\end{align*}
when we choose $h^\star = \lceil 2\log_{\epsilon / A} (32v/(3 \sqrt{3}) \rceil \leq H$.

Using the same intialization of datasets $\cD_h$ of \pref{alg:egreedy_genfun} as in the $S=2$ case, each copy $i$ of the MDP has probability $p = 3/8$ to be initialized with $(s_{g, i}, 1, 1, s_{g, i})$. By Hoeffding bound, the probability that at least $n/2$ copies are initialized in such way is at most
$\bbP(\sum_{i=1}^n X_i - pn > n/8) \leq \exp(-2n/64)$. Thus, with probability at least $1 - \exp(-2n/64) \geq 1/5$, there are at least $n/2$ copies $i$ which are initialized with reward $0$ for state-action pair $(s_{g, i}, 1)$. As in the $S = 2$ case, for each $i$, the agent needs to collect another sample from this state-action pair which only happens with probability $(\epsilon / A)^{h^\star}$ per copy. Therefore, the probability that the agent receives an informative sample for any of the suboptimal copies is bounded by $(\epsilon / A)^{h^\star}$ and the agent needs to collect at least $n/8$ samples before the greedy policy can become $1/16$-optimal. Hence, the expected number of times until this happen is at least
\begin{align*}
   \frac{1}{5} \cdot \frac{n}{8} \cdot (A / \epsilon)^{h^\star} = \Omega\left(\frac{S}{\alpha(\cF', \cF)^2}\right)
\end{align*}
which completes the proof.
\end{proof}

\section{Proofs for Regret and Sample Complexity Upper Bounds}
\label{app:upper_bound_proofs}

\thmsamplecomplex* 

\begin{proof}
We partition $\cF'$ into $\cF' = \cF'_1 \cup \dots \cup \cF'_{i_{\max}}$ with $\cF'_i = \{ f \in \cF' \colon c(f, \cF, \Pi_{\cF}, \myopic, \cM) \in [e^{i-1}, e^i]\}$ and $i_{\max} = \lceil \ln \sup_{f \in \cF'} c(f, \cF, \Pi_{\cF}, \myopic, \cM) \rceil$ and denote by $\cK_{i, T}  = \{ t \in [T] \colon \wh f_k \in \cF'_i\}$ the (random) set of episodes in $[T]$ where the Q-function estimate for the episode is in the $i$-th part.
To keep the notation concise, we denote for each $t \in \bbN$ 
\begin{itemize}
    \item $\pi_t$ as the greedy policy of $\wh f_t$, i.e., $\pi_t = \pi^{\wh f_t}$
    \item $\wt \pi_t$ as the exploration policy in episode $t$, i.e., $\wt \pi_k = \myopic(\wh f_t)$
    \item $\pi_t'$ as the improvement policy $\pi' \in \Pi_{\cF}$ that attains the maximum in the myopic exploration gap definition  for $\wh f_{t}$ (defined in \pref{eq:local_gap}). 
\end{itemize}
The total difference in return between the greedy and improvement policies can be bounded using \pref{lem:regret_decomp} as
\begin{align}\label{eqn:decomp1}
   \sum_{t \in \cK_{i, T}} \left(V_1^{\pi_t'}(x_{1}) - V_1^{\pi_t}(x_{1})\right)
    \leq
        \sum_{t \in \cK_{i, T}}  \sum_{h=1}^H \bbE_{\pi_t}[(\cE_h \wh f_{t})(x_{h}, a_{h})]
    - \sum_{t \in \cK_{i, T}} \sum_{h=1}^H \bbE_{\pi'_t}[(\cE_h \wh f_{t})(x_{h}, a_{h})]
\end{align} 
Using the completeness assumption in  \pref{ass:completeness}, we show in \pref{lem:squared_error_bound}
that with probability at least $1 - \delta $ for all $(h, t) \in [H] \times \bbN$
\begin{align*}
        \sum_{\tau = 1}^{t-1} \bbE_{\wt \pi_\tau}[(\cE_h^2 \wh f_{t})(x_h, a_h)] \leq
    3 \frac{t-1}{T} + 176 \ln \frac{6 N'_{\cF}(1/T)\ln(2t)}{\delta}
\end{align*}
where $N'_{\cF}(1/T) = \sum_{h=1}^H N_{\cF_h}(1/T) N_{\cF_{h+1}}(1/T)$.
In the following, we consider only the event where this condition holds.  
Leveraging the definition of $c(f, \cF, \Pi_{\cF}, \myopic, \cM)$, we bound
\begin{align*}
\sum_{\tau \in \cK_{i, t-1}} \bbE_{\pi'_\tau}[(\cE_h \wh f_{t})(x_h, a_h)]^2 &\leq 
\sum_{\tau \in \cK_{i, t-1}} \bbE_{\pi'_\tau}[(\cE_h^2 \wh f_{t})(x_h, a_h)] \tag{Jensen's inequality} \\
&\leq \sum_{\tau =1}^{t-1} \bbE_{\pi'_\tau}[(\cE_h^2 \wh f_{t})(x_h, a_h)]
\leq e^{i}
    \sum_{\tau =1}^{t-1}  \bbE_{\wt \pi_\tau}[(\cE_h^2 \wh f_{t})(x_h, a_h)]\\ &\leq 179 e^i \ln \frac{6 N'_{\cF}(1/T)\ln(2t)}{\delta}. 
\end{align*}

Using the distributional Eluder dimension machinery in \pref{lem:bedim}, this implies that
\begin{align*}
\sum_{t \in \cK_{i, T}} |\bbE_{\pi'_t}[(\cE_h \wh f_{t})(x_h, a_h)]|
&\leq 
        O \Bigg( 
    \sqrt{ e^{i} d(\cF'_i)
    \ln \frac{ N'_{\cF}(1/T)\ln(T)}{\delta}
    |\cK_{i, T}| } 
     + \min\{|\cK_{i, T}|, d(\cF'_i) \}\Bigg)
\end{align*}
where $d(\cF') = \operatorname{dim}_{BE}(\cF', \Pi_{\cF'}, T^{-1/2}) = \max_{h \in [H]} \operatorname{dim}_{DE}(\cF'_h - \cK_h \cF'_h , \Pi_{\cF'}, T^{-1/2})$ is the Bellman-Eluder dimension.
Applying the arguments above verbatim, we can derive the same upper-bound for $\sum_{t \in \cK_{i, T}} |\bbE_{\pi_t}[(\cE_h^2 \wh f_{t})(x_h, a_h)]|$. Plugging the above two bounds in \pref{eqn:decomp1}, we obtain
\begin{align*}
    \sum_{t \in \cK_{i, T}} \left[ V^{\pi_t'}_1(x_{1}) - V_1^{\pi_t}(x_{1}) \right]
   & \leq         O \Bigg( 
    \sqrt{ e^{i} H^2 d(\cF'_i)
    \ln \frac{N'_{\cF}(1/T)\ln(T)}{\delta}
    |\cK_{i, T}| } 
     + H  d(\cF'_i) \Bigg)~.
\end{align*}
Using the myopic exploration gap in \pref{def:local_exploration_gap}, we lower-bound the LHS as
\begin{align*}
    \sum_{t \in \cK_{i, T}} \left[ V^{\pi_t'}_1(x_{1}) - V_1^{\pi_t}(x_{1}) \right]
    \geq |\cK_{i, T}| \sqrt{e^{i-1}} \inf_{f \in \cF'_i} \alpha(f, \cF, \Pi_{\cF}, \myopic, \cM)
\end{align*}
Combining both bounds and rearranging yields
\begin{align*}
    |\cK_{i, T}|
   & \leq         O \Bigg( 
    \sqrt{ \frac{H^2 d(\cF'_i)}{\inf_{f \in \cF'_i} \alpha(f, \cF, \Pi_{\cF}, \myopic, \cM)^2}
    \ln \frac{N'_{\cF}(1/T)\ln(T)}{\delta}
    |\cK_{i, T}| } 
     + \frac{H d(\cF'_i)}{\inf_{f \in \cF'_i} \alpha(f, \cF, \Pi_{\cF}, \myopic, \cM)} \Bigg).
\end{align*}
We apply the AM-GM inequality and rearrange terms to arrive at 
\begin{align*}
        |\cK_{i, T}|
   & \leq        O \Bigg( 
     \frac{H^2 d(\cF'_i)}{\inf_{f \in \cF'_i} \alpha(f, \cF, \Pi_{\cF}, \myopic, \cM)^2}
    \ln \frac{N'_{\cF}(1/T)\ln(T)}{\delta} \Bigg).
\end{align*}
Taking a union bound over $i \in [i_{\max}]$ and summing the previous bound gives
\begin{align*}
        \sum_{t=1}^T \indicator{\wh f_t \in \cF'} &=     \sum_{i=1}^{i_{\max}} |\cK_{i, T}|
    \leq        O \Bigg( 
     \sum_{i=1}^{i_{\max}} \left(\frac{d(\cF'_i)}{\inf_{f \in \cF'_i} \alpha(f, \cF, \Pi_{\cF}, \myopic, \cM)^2}\right)
    H^2 \ln \frac{N'_{\cF}(1/T)\ln(T)}{\delta} \Bigg)\\
    & \leq 
    O \Bigg( 
     \frac{H^2 d(\cF')}{\inf_{f \in \cF'} \alpha(f, \cF, \Pi_{\cF}, \myopic, \cM)^2}
    \ln(\sup_{f \in \cF'} c(f, \cF, \Pi_{\cF}, \myopic, \cM)) \ln \frac{ N'_{\cF}(1/T)\ln(T)}{\delta} \Bigg).
\end{align*}
Finally, since $N'_{\cF}(1/T) = \sum_{h=1}^H N_{\cF_{h}}(1/T) N_{\cF_{h+1}}(1/T) \leq \left(\sum_{h=1}^H N_{\cF_{h}}(1/T) \right)^2 = \bar N_{\cF}(1/T)^2$, the statement follows.
\end{proof}

\begin{customlemma}{3}\label{eight}
Let $f = \{f_h\}_{h \in [H]}$ with $f_h \colon \cX \times \cA \rightarrow \bbR$ and $\pi^f$ be the greedy policy of $f$. Then for any policy $\pi' \in \Pi$,
\vspace{-1mm}
\begin{align*}
    &V^{\pi'}_1(x_{1}) - V^{\pi^f}_1(x_{1}) 
    \leq 
    \sum_{h=1}^H \bbE_{\pi^f}[(\cE_{h}f)(x_{h}, a_{h})]
    - \sum_{h=1}^H \bbE_{\pi'}[(\cE_{h}f)(x_{h}, a_{h})].
\end{align*}
\end{customlemma}
\begin{proof}
For any function $g \colon \cX \rightarrow \bbR$, we define $(\cB_h g)(x,a) = \bbE\left[r_h + g(x_{h+1}) ~|~x_h = x, a_h = a\right]$.
First, we write the difference in value functions at any state $x \in \cX$ and time $h \in [H]$ as
\begin{align}
    V_h^{\pi'}(x) - V_h^{\pi}(x)
    &=
    \bbE[Q_h^{\pi'}(x, a) ~|~a \sim \pi'_h(x)] - 
    Q_h^{\pi}(x, \pi_{h}(x))\nonumber \\
    &=
        \bbE[Q_h^{\pi'}(x, a) - f_{h}(x, a) ~|~a \sim \pi'_h(x)]
        + f_{h}(x, \pi_{h}(x)) 
        - 
    Q_h^{\pi}(x, \pi_{h}(x))
    \nonumber \\
    & \qquad + \underset{\leq 0}{\underbrace{\bbE[f_{h}(x, a) ~|~a \sim \pi'_h(x)] 
    - f_{h}(x, \pi_{ h}(x))}}
    \label{eqn:Vdiffbound}
\end{align}
where the last term is non-positive because $\pi$ is the greedy policy of $f$. Let $g_h(x) = \max_{a \in \cA} f_h(x,a)$ for all $x \in \cX$ and write the difference $Q^{\pi'}_h - f_h$ as
\begin{align*}
    Q^{\pi'}_h - f_{h} 
    &= 
    \cB_h V^{\pi'}_{h+1}
    - f_h + \cB_h g_{h+1} - \cB_h g_{h+1}
    = \cB_h(V^{\pi'}_{h+1} - g_{h+1}) - f_h + \cT_h f_{h+1}\\
    &= \cB_h(V^{\pi'}_{h+1} - g_{h+1}) - (\cE_{h} f)(x_h, a_h) 
\end{align*}
where we used the linearity of $\cB$ and the fact that $\cB_h g_{h+1} = \cT_h f_{h+1}$.
For any $x \in \cX$ we can further bound
\begin{align*}
    V^{\pi'}_{h+1}(x) - g_{h+1}(x) \leq \bbE[Q^{\pi'}_{h+1}(x, a) - f_{h+1}(x,a) ~|~a \sim \pi'_h(x)] 
\end{align*}
and combining this with the previous identity, we have
\begin{align*}
\bbE_{\pi'}[Q_h^{\pi'}(x_h, a_h) - f_{h}(x_h, a_h) ~|~x_h = x] 
&\leq \bbE_{\pi'}[V^{\pi'}_{h+1}(x_{h+1}) - g_{h+1}(x_{h+1})  - (\cE_{h} f)(x_h, a_h)  ~|~x_h = x]~  \\ 
&\leq \bbE_{\pi'}[(Q^{\pi'}_{h+1}(x_{h+1}, a_{h+1}) - f_{h+1}(x_{h+1}, a_{h+1}))  - (\cE_{h} f)(x_h, a_h)  ~|~x_h = x].
\numberthis \label{eqn:piprime_bound}
\end{align*}
Similarly, we can write
\begin{align}
\begin{split}
    \bbE_{\pi}[(f_{h} - 
    Q_h^{\pi})(x_{h}, a_h) ~|~ x_h = x]
    &= \bbE_{\pi}[(f_h - \cT_h f_{h+1} + \cB_{h}  g_{h+1} - 
    \cB_h V_{h+1}^{\pi})(x_{h}, a_h) ~|~ x_h = x]\\
    &= \bbE_{\pi}[(\cE_{h} f)(x_h, a_h) + g_{h+1}(x_{h+1}) - V_{h+1}^{\pi}(x_{h+1} )~|~ x_h = x]\\
    &= \bbE_{\pi}[(\cE_{h} f)(x_h, a_h) + (f_{h+1}(x_{h+1}, a_{h+1}) - Q_{h+1}^{\pi}(x_{h+1}, a_{h+1})) ~|~ x_h = x]~.
    \end{split}
    \label{eqn:pi_bound}
\end{align}
Applying \pref{eqn:piprime_bound} and \pref{eqn:pi_bound} recursively to \pref{eqn:Vdiffbound}, we arrive at the desired statement
\begin{align*}
    V_1^{\pi'}(x_1) - V_1^{\pi}(x_1)
    &\leq \bbE_{\pi'}[Q_1^{\pi'}(x_1, a_1) - f_{1}(x_1, a_1) ]
    + \bbE_{\pi}[(f_{1} - 
    Q_1^{\pi})(x_{1}, a_1)]\\ 
    &\leq \sum_{h=1}^H \bbE_{\pi}[(\cE_h f)(x_h, a_h)]
    - \sum_{h=1}^H \bbE_{\pi'}[(\cE_h f)(x_h, a_h)].
\end{align*}
\end{proof}

\begin{restatable}{lemma}{squareerr}
\label{lem:squared_error_bound}
Consider \pref{alg:egreedy_genfun} with a function class $\cF$ that satisfies \pref{ass:completeness}.
Let $\rho \in \bbR^+$  and $\delta \in (0, 1)$. Then with probability at least $1 - \delta$ for all  $h \in [H]$ and $t \in \bbN$ 
\begin{align*}
    \sum_{\tau=1}^{t-1} \bbE_{\myopic(\wh f_\tau)}[ (\cE_h^2 \wh f_t) (x_h, a_h)] 
    &\leq
    3 \rho t + 176 \ln \frac{6 N'_{\cF}(\rho)\ln(2t)}{\delta}
\end{align*}
where $N'_{\cF}(\rho) = \sum_{h=1}^H N_{\cF_h}(\rho) N_{\cF_{h+1}}(\rho)$ is the sum of  $\ell_\infty$ covering number of $\cF_{h} \times \cF_{h+1}$ w.r.t.  radius $\rho>0$.
\end{restatable} 

\begin{proof}
The proof closely follows the proof of Lemma~39 by \citet{jin2021bellman}.
We first consider a fixed $t \in \bbN, h \in [H]$ and $f = \{f_h, f_{h+1}\}$ with $f_h \in \cF_h, f_{h+1} \in \cF_{h+1}$. Let 
\begin{align*}
    Y_{t, h}(f) & =\quad (f_h(x_{t, h}, a_{t, h}) - r_{t, h} - \max_{a'} f_{h+1}(x_{t, h+1}, a'))^2 - ((\cT_h f_{h+1})(x_{t, h}, a_{t, h}) - r_{t, h} - \max_{a'} f_{h+1}(x_{t, h+1}, a'))^2\\
    &= (f_h(x_{t, h}, a_{t, h}) - (\cT_h f_{h+1})(x_{t, h}, a_{t, h})) \\
    & \qquad \times (f_h(x_{t, h}, a_{t, h}) + (\cT_h f_{h+1})(x_{t, h}, a_{t, h}) - 2r_{t, h} - 2\max_{a'} f_{h+1}(x_{t, h+1}, a'))
\end{align*}
and let $\mathfrak{F}_t$ be the $\sigma$-algebra under which all the random variables in the first $t-1$ episodes are measurable. Note that $|Y_{t, h}(f)| \leq 4$ almost surely and the conditional expectation of $Y_{t, h}(f)$ can be written as
\begin{align*}
    \bbE[ Y_{t, h}(f) ~|~ \mathfrak{F}_{t}]
    &= \bbE[\bbE[ Y_{t, h}(f) ~|~ \mathfrak{F}_{t}, x_{t, h}, a_{t, h}]~|~\mathfrak{F}_{t}] = \bbE_{\myopic(\wh f_t)}[ 
    (f_h - \cT_h f_{h+1})(x_{h}, a_{h})^2]~.
\end{align*}
The variance is bounded as
\begin{align*}
\operatorname{Var}[ Y_{t, h}(f) ~|~ \mathfrak{F}_{t}]
\leq  \bbE[Y_{t, h}(f)^2 ~|~ \mathfrak{F}_{t}]
\leq 16 \bbE[(f_h - \cT_h f_{h+1})(x_{t, h}, a_{t, h})^2~|~ \mathfrak{F}_{t}]
= 16 \bbE[ Y_{t, h}(f) ~|~ \mathfrak{F}_{t}]
\end{align*}
since
$|f_h(x_{t, h}, a_{t, h}) + (\cT_h f_{h+1})(x_{t, h}, a_{t, h}) - 2r_{t, h} - 2\max_{a'} f_{h+1}(x_{h+1}, a')| \leq 4$ almost surely.
Applying \pref{lem:simplified_freedman} to the random variable $Y_{t, h}(f)$, we have that with probability at least \(1 - \delta\), for all $t \in \bbN$, 
\begin{align*}
   \sum_{i=1}^t  \bbE[ Y_{i, h}(f) ~|~ \mathfrak{F}_{i}] 
   &\leq 2 A_t \sqrt{\sum_{i=1}^t  \operatorname{Var}[ Y_{i, h}(f) ~|~ \mathfrak{F}_{i}]}
    + 12 A_t^2 +  \sum_{i=1}^t  Y_{i, h}(f)\\
    &\leq 8 A_t \sqrt{\sum_{i=1}^t  \bbE[ Y_{i, h}(f) ~|~ \mathfrak{F}_{i}]}
    + 12 A_t^2 +  \sum_{i=1}^t  Y_{i, h}(f)~,
\end{align*}
where $A_t = \sqrt{2 \ln \ln(2 t) + \ln (6 / \delta)}$.
Using AM-GM inequality and rearranging terms in the above, we get that 
\begin{align*}
   \sum_{i=1}^t  \bbE[ Y_{i, h}(f) ~|~ \mathfrak{F}_{i}] 
    &\leq  2\sum_{i=1}^t  Y_{i, h}(f) + 88 A_t^2
    \leq 2\sum_{i=1}^t  Y_{i, h}(f) + 176 \ln \frac{6\ln(2t)}{\delta}~.
\end{align*}
Let $\cZ_{\rho, h}$ be a $\rho$-cover of $\cF_{h} \times \cF_{h+1}$. Now taking a union bound over all $\phi_h \in  \cZ_{\rho, h}$ and $h \in [H]$, we obtain that with probability at least $1 - \delta$ for all $\phi_h$ and $h \in [H]$
\begin{align*}
   \sum_{i=1}^t  \bbE[ Y_{i, h}(\phi_h) ~|~ \mathfrak{F}_{i}] 
    &\leq 2\sum_{i=1}^t  Y_{i, h}(\phi_h) + 176 \ln \frac{6 N'_{\cF}(\rho) \ln(2t)}{\delta}~.
\end{align*}
This implies that with probability at least \(1 - \delta\), for all $f = \{f_h, f_{h+1}\} \in  \cF_h \times \cF_{h+1}$ and $h \in [H]$, 
\begin{align*}
   \sum_{i=1}^t \bbE[ Y_{i, h}(f) ~|~ \mathfrak{F}_{i}] 
    &\leq 2\sum_{i=1}^t Y_{i, h}(f) + 3 \rho (t-1)  + 176 \ln \frac{6 N_{\cF}(\rho)\ln(2t)}{\delta}~.
\end{align*}
This holds in particular for $f = \wh f_t = \{\wh f_{t, h}, \wh f_{t, h+1}\}$ for all $t \in \bbN$.
Finally, we have 
\begin{align*}
    \sum_{i=1}^{t-1} Y_{i, h}(\wh f_t)
    &= \sum_{i=1}^{t-1} (\wh f_{t, h}(x_{i, h}, a_{i, h}) - r_{i, h} - \max_{a'} \wh f_{t, h+1}(x_{i, h+1}, a'))^2 \\ 
    &\qquad \qquad \qquad - \sum_{i=1}^{t-1}((\cT_h \wh f_{t, h+1})(x_{i, h}, a_{i, h}) - r_{i, h} - \max_{a'} \wh f_{t,h+1}(x_{i, h+1}, a'))^2\\
    &=  \inf_{f' \in \cF_h}\sum_{i=1}^{t-1} (f'(x_{i, h}, a_{i, h}) - r_{i, h} - \max_{a'} \wh f_{t, h+1}(x_{i, h+1}, a'))^2\\
    &\qquad \qquad \qquad - \sum_{i=1}^{t-1}((\cT_h \wh f_{t, h+1})(x_{i, h}, a_{i, h}) - r_{i, h} - \max_{a'} \wh f_{t,h+1}(x_{i, h+1}, a'))^2 \\
    &\leq 0~, 
\end{align*}
where the final inequality follows from completeness in \pref{ass:completeness}.
Therefore, we have with probability at least $1 - \delta$ for all $t \in \bbN $ 
\begin{align*}
       \sum_{i=1}^{t-1}  \bbE_{\myopic(\wh f_i)}[ 
    (\wh f_{t, h} - \cT_h \wh f_{t, h+1})(x_{h}, a_{h})^2] 
    &\leq 3 \rho (t-1)  + 176 \ln \frac{6H N'_{\cF}(\rho)\ln(2t)}{\delta}~.
\end{align*}
\end{proof} 

\egreedyregret* 
\begin{proof}
We decompose the regret as
\begin{align*}
    \operatorname{Reg}(T) &= \sum_{t=1}^T \left( V^\star(x_{t, 1}) - V^{\myopic(\wh f_t)}(x_{t, 1})\right)\\
    &= \sum_{t=1}^T \left( V^\star(x_{t, 1}) -  V^{\pi_t}(x_{t, 1})\right)  +  \sum_{t=1}^T \left(V^{\pi_t}(x_{t, 1}) - V^{\myopic(\wh f_t)}(x_{t, 1}) \right)~,
\end{align*}
and bound both terms individually. The excess regret due to exploration in the second term is bounded as
\begin{align*}
      \sum_{t=1}^T \left(V^{\pi_t}(x_{t, 1}) - V^{\myopic(\wh f_t)}(x_{t, 1}) \right)
      \leq T H \epsilon .
\end{align*}
Second, let $\cF_i = \{f \in \cF \colon V^\star(x_{t, 1}) - V^{\pi^f}(x_{t, 1}) \in [(1/2)^{i}, (1/2)^{i-1}] \}$ the value functions that incur regret $[(1/2)^{i}, (1/2)^{i-1}]$ per episode. Applying \pref{thm:sample_complexity} above, we get
\begin{align*}
    \sum_{t=1}^T \left( V^\star(x_{t, 1}) - V^{\pi_t}(x_{t, 1})\right)
    \leq T 2^{-m} + \sum_{i=1}^m 2^{-i} O \Bigg( 
     \frac{H^2 d}{\legap(\cF_i, \cF)^2}
    \ln(\lerad(\cF_i, \cF)) \ln \frac{m \bar N_{\cF_i}(1/T)\ln(T)}{\delta} \Bigg)~.
\end{align*}
for any $m \in \bbN$.
For $\epsilon$-greedy, we can bound $\legap(\cF_i, \cF) \leq (A / \epsilon)^H$ and thus
\begin{align*}
    \operatorname{Reg}(T)
    \leq TH \epsilon + T 2^{-m} + H^3 d \ln(A / \epsilon) \ln \frac{m \bar N_{1/T}(\cF)\ln(T) }{\delta} \sum_{i=1}^m 2^{-i}  O \Bigg( 
     \frac{1}{\legap(\cF_i, \cF)^2}
     \Bigg)~.
\end{align*}
Assume $\legap(\cF_i, \cF) \geq (\epsilon / A)^{h/2} 2^{-i-1}$, which which always holds for $h = H$. Then
\begin{align*}
    \sum_{i=1}^m 2^{-i}   
     \frac{1}{\legap(\cF_i, \cF)^2}
     \leq \frac{A^h}{\epsilon^h} \sum_{i=1}^m \frac{2^{-i}}{2^{-2i-2}}
     = \frac{A^h}{\epsilon^h} \sum_{i=1}^m 2^{i+2} \leq \frac{4A^h m}{\epsilon^h} 2^m 
\end{align*}
and setting $m$ such that $2^m = \sqrt{\frac{T}{m d H^3 \ln(A / \epsilon) \ln \frac{\bar N_{\cF}(1/T) \ln(T)}{\delta}}} (\epsilon / A)^{h/2}$ gives
\begin{align*}
    \operatorname{Reg}(T)
    \leq TH \epsilon + \tilde  O \Bigg(
    \sqrt{\frac{h d H^3 A^h T}{\epsilon^h}  \ln(A) \ln(T) \ln \frac{\bar N_\cF(1/T)}{\delta}}
     \Bigg).
\end{align*}
Finally, denote by $L = \tilde O( \ln(A) \ln(T) \ln \frac{\bar N_\cF(1/T)}{\delta})$ the log-terms above and assume the exploration parameter is chosen as 
\begin{align*}
\epsilon = \Theta\left(\left(\frac{h H d A^h L}{T} \right)^{\frac{1}{2 + h}}\right).
\end{align*}
Then the regret bound evaluates to
\begin{align*}
    \operatorname{Reg}(T)
    \leq \tilde O\left( H T^{\frac{h+1}{h+2}} \left(H h A^h d \ln(A) \ln \frac{\bar N_\cF(1/T)}{\delta}\right)^{\frac{1}{h+2}}\right).
\end{align*}
\end{proof}

\section{Supporting Technical Results}
\label{app:supporting_results}

We recall the following standard definitions. 

\begin{definition}[$\varepsilon$-independence between distributions]
Let $\cG$ be a class of functions defined on a space $\cX$, and
$\nu, \mu_1,\dots,\mu_n$ be probability measures over $\cX$. We say $\nu$ is $\varepsilon$-independent of $\{\mu_1, \mu_2, \dots , \mu_n\}$ with respect to $\cG$ if there exists $g \in \cG$ such that $\sqrt{\sum_{i=1}^n(\bbE_{\mu_i}[g])^2}\leq \varepsilon$, but $|\bbE_\nu[g]| > \varepsilon$.
\end{definition}

\begin{definition}[(Distributional Eluder (DE) dimension]
Let $\cG$ be a function class defined on $\cX$ , and
$\Pi$ be a family of probability measures over $\cX$ . The distributional Eluder dimension $\dim_{\operatorname{DE}}(\cG, \Pi, \varepsilon)$
is the length of the longest sequence $\{\rho_1, \dots , \rho_n\} \subset \Pi$ such that there exists $\varepsilon'\geq \varepsilon$ where $\rho_i$ is $\varepsilon'$-independent of $\{\rho_1, \dots, \rho_{i-1}\}$ for all $i \in [n]$.
\end{definition}

\begin{definition}[Bellman Eluder (BE) dimension \citep{jin2021bellman}]
Let $\cE_h \cF$ be the
set of Bellman residuals induced by $\cF$ at step $h$, and $\Pi = \{\Pi_h\}_{h=1}^H$ be a collection of $H$ probability
measure families over $\cX \times \cA$. The $\varepsilon$-Bellman Eluder dimension of $\cF$ with respect to $\Pi$ is defined as
\begin{align*}
    \dim_{\operatorname{BE}}(\cF,\Pi,\varepsilon) := \max_{h\in[H]}\dim_{\operatorname{DE}}(\cE_h \cF,\Pi,\varepsilon)\,.
\end{align*}
\end{definition}

\begin{lemma}[Lemma~41, \citet{jin2021bellman}]
\label{lem:bedim}
Given a function class $\Phi$ defined on $\cX$ with $|\phi(x)| \leq C$ for all $(\phi, x) \in \Phi \times \cX$ and a family of probability measures $\Pi$ over $\cX$. Suppose sequences $\{ \phi_i\}_{i \in [K]} \subseteq \Phi$ and $\{ \mu_i\}_{i \in [K]} \subseteq \Pi$ satisfy for all $k \in [K]$ that $\sum_{i = 1}^{k-1} (\bbE_{\mu_i}[\phi_k])^2 \leq \beta$. Then for all $k \in [K]$ and $\omega > 0$
\begin{align*}
    \sum_{t = 1}^k | \bbE_{\mu_t}[\phi_t]| 
    \leq
    O \left( 
    \sqrt{\operatorname{dim}_{DE}(\Phi, \Pi, \omega) \beta k} 
    + \min\{k,  \operatorname{dim}_{DE}(\Phi, \Pi, \omega)\}C + k \omega\right)
\end{align*}
\end{lemma}

\begin{lemma}[Time-Uniform Freedman Inequality]
\label{lem:simplified_freedman}
Suppose $\{ X_t \}_{t=1}^\infty$ is a martingale difference sequence with $| X_t | \leq b$. Let 
\begin{equation*}
    \mathrm{Var}_\ell(X_\ell) = \mathbf{Var}( X_\ell | X_1, \cdots, X_{\ell-1})
\end{equation*}
Let $V_t = \sum_{\ell=1}^t \mathrm{Var}_\ell(X_\ell)$ be the sum of conditional variances of $X_t$.  Then we have that for any $\delta' \in (0,1)$ and $t \in \mathbb{N}$
\begin{equation*}
    \mathbb{P}\left(  \sum_{\ell=1}^t X_\ell >    2\sqrt{V_t}A_t + 3b A_t^2 \right) \leq \delta'
\end{equation*}
Where $A_t = \sqrt{2 \ln \ln \left(2 \left(\max\left(\frac{V_t}{b^2} , 1\right)\right)\right) + \ln \frac{6}{\delta'}}$.
\end{lemma}
\begin{proof}
See \citet{howard2021time}.
\end{proof}

\end{document}